\documentclass[12pt]{article}
\usepackage[dvips]{graphicx}
\usepackage{amssymb,amsmath,color}
\usepackage{url}

\usepackage[utf8]{inputenc} 
\usepackage[T1]{fontenc}    
\usepackage{hyperref}       
\usepackage{url}            
\usepackage{amssymb,amsmath,color}
\usepackage{booktabs}       
\usepackage{amsfonts}       
\usepackage{nicefrac}       
\usepackage{microtype}      
\usepackage{graphicx}

\usepackage[mathcal]{eucal}
\newcommand{\BEAS}{\begin{eqnarray*}}
\newcommand{\EEAS}{\end{eqnarray*}}
\newcommand{\BEA}{\begin{eqnarray}}
\newcommand{\EEA}{\end{eqnarray}}
\newcommand{\BEQ}{\begin{equation}}
\newcommand{\EEQ}{\end{equation}}
\newcommand{\BIT}{\begin{itemize}}
\newcommand{\EIT}{\end{itemize}}
\newcommand{\BNUM}{\begin{enumerate}}
\newcommand{\ENUM}{\end{enumerate}}
\newcommand{\BA}{\begin{array}}
\newcommand{\EA}{\end{array}}

\newcommand{\rb}{\mathbb{R}}
\newcommand{\BlackBox}{\rule{1.5ex}{1.5ex}}  

\newenvironment{proof}{\par\noindent{\bf Proof\ }}{\hfill\BlackBox\\[2mm]}
\newtheorem{lemma}{Lemma}

\newtheorem{proposition}{Proposition}

\newcommand{\mysec}[1]{Section~\ref{sec:#1}}
\newcommand{\eq}[1]{Eq.~(\ref{eq:#1})}
\newcommand{\myfig}[1]{Figure~\ref{fig:#1}}

\def \E{{\mathbb E}}

\def \S{  { \mathcal{S}} }
\def \A{  { \mathcal{A}} }
\def \W{  { \mathcal{W}} }
\def \Z{  { \mathcal{Z}} }

\def \E{{\mathcal{E}}}

\let\OLDthebibliography\thebibliography
\renewcommand\thebibliography[1]{
  \OLDthebibliography{#1}
  \setlength{\parskip}{2.4pt}
  \setlength{\itemsep}{0pt plus 0.3ex}
}

\title{Max-Plus Matching Pursuit for Deterministic Markov Decision Processes}

\author{
  Francis Bach \\
  Inria  \\
  D\'epartement d'Informatique de l'Ecole Normale Sup\'erieure \\
  PSL Research University, Paris, France \\
  \texttt{francis.bach@ens.fr}
  }

\usepackage{enumitem}
\setlist[itemize]{leftmargin=7mm}

\begin{document}

\maketitle

\begin{abstract}
We consider deterministic Markov decision processes (MDPs) and   apply \emph{max-plus} algebra tools to approximate the value iteration algorithm by a smaller-dimensional iteration based on a representation on dictionaries of value functions.
  The set-up naturally leads to novel theoretical results which are  simply formulated due to the max-plus algebra structure. For example, when considering a fixed (non adaptive) finite basis, the computational complexity of approximating the optimal value function is not directly related to the number of states, but to notions of covering numbers of the state space.
 In order to break the curse of dimensionality in factored state-spaces, we consider adaptive basis that can adapt to particular problems leading to an algorithm similar to matching pursuit from signal processing. They currently come with no theoretical guarantees but work empirically well on simple deterministic MDPs derived from low-dimensional continuous control problems.
We focus primarily on deterministic MDPs but note  that the framework can be applied to all MDPs by considering measure-based formulations.
 \end{abstract}

\section{Introduction}
 
Function approximation for Markov decision processes (MDPs) is an important problem in reinforcement learning. Simply extending classical representations from supervised learning  is not straightforward because of the specific non-linear structure of MDPs; for example the linear parametrization of the value function is not totally adapted to the algebraic structure of the Bellman operator which is central in their analysis and involves ``max'' operations.
 
Following~\cite{mceneaney2002error,akian2008max} which applied similar concepts to  problems in optimal control, we consider a different semi-ring than the usual ring $(\rb,+,x)$, namely the \emph{max-plus} semi-ring $(\rb \cup \{-\infty\}, \oplus, \otimes) = (\rb \cup \{-\infty\}, \max, +)$.   The new resulting algebra is natural for MDPs, as for example for deterministic discounted MDPs, the Bellman operator happens to be additive and positively homogeneous for the max-plus algebra.

  In this paper, we explore classical concepts in linear representations in machine learning and signal processing, namely approximations  from a finite basis and sparse approximations through greedy algorithms such as matching pursuit~\cite{mallat2008wavelet,mairal2014sparse,temlyakov2015sparse,dumitrescu2018dictionary},  explore them for the max-plus algebra and apply it to deterministic MDPs with known dynamics (where the goal is to estimate the optimal value function). We make the following contributions, after briefly reviewing MDPs in \mysec{mdp}:

  \BIT
  \item In \mysec{maxplus}, we apply \emph{max-plus} algebra tools to approximate the value iteration algorithm by a smaller-dimensional iteration based on a representation on dictionaries of value functions.
  \item As shown in \mysec{approx}, 
  the set-up naturally leads to novel theoretical results which are  simply formulated due to the max-plus algebra structure. For example, when considering a fixed (non adaptive) finite basis, the computational complexity of approximating the optimal value function is not directly related to the number of states, but to notions of covering numbers.
  
 \item  In \mysec{greedy}, in order to circumvent the curse of dimensionality in factored state-spaces, we consider adaptive basis that can adapt to particular problems leading to an algorithm similar to matching pursuit. It currently comes with no theoretical guarantees but works empirically well in \mysec{exp} on simple deterministic MDPs derived from low-dimensional  control problems.

 \EIT

\section{Markov Decision Processes}
\label{sec:mdp}
 We consider a standard MDP~\cite{puterman2014markov,sutton2018reinforcement}, defined by a finite state space $\mathcal{S}$, a finite action space~$\A$, a reward function $r: \S \times \A \to \rb$, transition probabilities $p(\cdot|\cdot, \cdot): \S \times \S \times \A \to \rb$, and a discount factor $\gamma \in   [0,1)$.
 In this paper, we focus on the goal of finding (or approximating) the optimal value function $V_\ast: \S \to \rb$, from which the optimal policy $\pi_\ast: \S \to \A$, that leads to the optimal sum of discounted rewards, can be obtained as $
\pi_\ast(s) \in \arg \max_{a \in \A} r(s,a) + \gamma \sum_{s' \in \S} p(s'|s,a) V_\ast(s')$. We assume that the transition probabilities and the reward function are known.

We denote by $T$ the Bellman operator from $\rb^\S$ to $\rb^\S$ defined as, for a function $V: \S \to \rb$,
\[ \textstyle
TV(s) =  \max_{a \in \A} r(s,a) + \gamma \sum_{s' \in \S} p(s'|s,a) V(s').
\]
The optimal value function $V_\ast$ is the unique fixed point of $T$.
In order to find an approximation of~$V_\ast$, we simply need to find $V$ such that $\| TV - V\|_\infty$ is small, as  $\| V_\ast - V\|_\infty \leqslant   (1-\gamma)^{-1} \| V - TV \|_\infty$ (see proof in App.~\ref{app:lemmaTVV} taken from~\cite[Prop.~2.1]{bertsekas2012weighted}).

\paragraph{Value iteration algorithm.}
The usual value iteration algorithm considers the recursion $V_t \!=\! T V_{t-1}$, which converges exponentially fast to $V_\ast$ if $\gamma <1$~\cite{sutton2018reinforcement}. 
More precisely, if ${\rm range}(r)$ is defined as ${\rm range}(r)  \!=\! \max_{s \in \S} \max_{a\in \A}  {r}(s,a) \!- \! \min_{s \in \S} \max_{a \in \A}  {r}(s,a)$, and we initialize at $V_0\!=\! 0$, we reach precision $   {{\rm range}(r) }\varepsilon ({1-\gamma} )^{-1}$ after at most $t= {\log ({1}/{\varepsilon})}({\log ( {1}/{\gamma})})^{-1}
\leqslant   {\log ( {1}/{\varepsilon})}{(1-\gamma)^{-1}} $ iterations~\cite{sutton2018reinforcement}. In this paper, we consider discount factors which are close to $1$, and the term dependent on $ (\log ({1}/{\gamma}) )^{-1}$ or $(1-\gamma)^{-1}$ will always be the leading one---this thus excludes from consideration sampling-based algorithms with a better complexity in terms of $|\S|$ and $|\A|$ but worse dependence on the discount factor $\gamma$ (see, e.g.,~\cite{sidford2018variance,kakade2018variance} and references therein). Throughout this paper, we are going to refer to $\tau = (1-\gamma)^{-1}$ as the \emph{horizon} of the MDP (this is the expectation of a random variable with geometric distribution proportional to powers of $\gamma$), which characterizes the expected number of steps in the future that need to be taken into account for computing rewards.

\paragraph{Deterministic MDPs.} In this paper, we consider \emph{deterministic MDPs}, i.e., MDPs for which given a state $s$ and an action $a$, a deterministic state $s'$ is reached. For these  MDPs, choosing an action is equivalent to choosing the reachable state $s'$. Thus, the transition behavior if fully characterized by an     edge set $\E \subset \S \times \S$, and we obtain the resulting reward function $
\bar{r}: \S \times \S \to \rb$, where $\bar{r}(s,s')$ is the maximal reward from actions leading from state $s$ to state $s'$, which is defined to be $-\infty$ if $(s,s') \notin \E$. The Bellman operator then takes the form
\[ \textstyle TV(s) = \max_{s' \in \S} \bar{r}(s,s') + \gamma V(s').\] 
 We focus primarily on deterministic MDPs but note in \mysec{conclusion} that the framework can be applied to all MDPs by considering a measure-based formulation~\cite{hernandez2012discrete,lasserre2008nonlinear}.

\paragraph{Factored state-spaces.} 
We also consider large state spaces (typically coming from the discretization of control problems). A classical example will be factored state-spaces where $\S = \S_1\times \cdots \times \S_d$, but we do not assume in general factorized dependences of the reward function on certain variables like in factored MDPs~\cite{guestrin2003efficient}.

\section{Max-plus Algebra applied to Deterministic MDPs}
\label{sec:maxplus}
Many works consider a regular linear parameterization of the value function~(see, e.g.,~\cite{geist2010brief} and references therein), as $V(s)  \!=\! \sum_{w \in \W} \alpha(w)  w(s)$ where~$\W$ is a set of basis functions $w: \S \to \rb$. Following~\cite{akian2008max}, we consider  max-plus linear combinations. For more general properties of max-plus algebra, see, e.g.,~\cite{baccelli1992synchronization,butkovivc2010max,cohen2004duality}.

\subsection{Max-plus-linear operators and max-plus-linear combinations}

In this section, we consider algebraic properties of our problem within the max-plus-algebra.
For deterministic MDPs, the key property is that the Bellman operator is additive and max-plus-homogeneous, that is, (a) $T (V + c) = TV + \gamma c$ for any constant $c$, which can be rewritten as
$T ( c \otimes V) = c^{\otimes \gamma} T V$ (where the equality $ \gamma c = c^{\otimes \gamma} $ for $\gamma \in \rb_+$, is an extension of the relationship $2c = c \otimes c$),
and (b) $T(\max\{V,V'\}) = \max \{ TV, TV'\}$, which can be rewritten as $T ( V \oplus V') = TV \oplus TV'$, where all operations are taken element-wise for all $s \in \S$. 

We  now explore various max-plus approximation properties~\cite{cohen2004duality}.
First,  regular linear combinations become:
$
V(s)  = \bigoplus_{w \in \W} \alpha(w) \otimes  w(s) = \max_{ w \in \W} \alpha(w) + w(s)
$.
We introduce the notation 
\BEQ
 \textstyle \W \alpha (s) =  \max_{ w \in \W} \alpha(w) + w(s),
 \EEQ
 so that the equation above may be rewritten as $V =  \W \alpha$.  Inverses of  max-plus-linear operator are typically not defined, but  a weaker notion of pseudo-inverse (often called ``residuation''~\cite{baccelli1992synchronization}) can be defined due to the idempotence of   $\otimes$. We thus consider the operator $\W^+ : \rb^\S \to \rb^\W$ defined as 
 \BEQ
  \textstyle \W^+ V(w) = \min_{s \in \S} V(s) - w(s).
  \EEQ
  We have, for the pointwise partial order on $\rb^\S$, $\W \alpha \leqslant V \Leftrightarrow
 \alpha \leqslant \W^+ V$, that is:
\[
\forall s \in \S, \ \W \alpha(s)  \!\leqslant \! V(s) \Leftrightarrow 
\forall (s,w) \in \S \times \W, \ \alpha(w) + w(s) \leqslant V(s)
\Leftrightarrow
\forall w \in \W, \ \alpha(w) \!\leqslant\W^+ V(w).    \]
Moreover, as shown in~\cite{cohen2004duality,akian2008max}, $\W \W^+ \W = \W $ and $\W^+ \W \W^+ = \W^+$, and thus $\W^+$ plays a role of pseudo-inverse, and $\W \W^+$ a role of projection on the image of $\W$; moreover, $\W\W^+ V \leqslant V$ for all~$V$, and $\W\W^+$ is idempotent and non-expansive for the $\ell_\infty$-norm. 

One approximation algorithm would be to replace $V_{t+1} = TV_t$ by
$
V_{t+1}= \W \W^+ T V_t
$, that is, if we consider $V_t$ of the form $V_t = \W \alpha_t$, then $V_{t+1}$ would be of the form
$\W \alpha_{t+1}$ where 
\[
\textstyle
\alpha_{t+1}(w) = \W^+ T \W \alpha_t (w) = \min_{s \in \S}  \max_{w' \in \W} \gamma \alpha_t(w') + T w'(s) - w(s),
\]
which requires to solve at each iteration an infimum problem over $\S$, which is computionally expensive
as $O(|\S| \cdot |\W|)$, which is typically worse than $O(|\E|)$ for classical value iteration. We are thus looking for an extra approximation that will lead to a decomposition where after some compilation, the iteration complexity is  only dependent on the number of basis functions.

\subsection{Max-plus transposition}
An important idea from~\cite{akian2008max} is to first project onto a different low-dimensional different image, with a projection which is efficient. In the regular functional linear setting, this would be equivalent to imposing equality only for certain efficient measurements.
We thus define, given a set $\Z$ of functions~$z$ from $\S$ to $\rb$ (which can be equal to $\W$ or not),
\BEQ
\textstyle
{{\Z\!}^\top} V(z) = \max_{s \in \S} V(s) + z(s).
\EEQ
The notation ${{\Z\!}^\top} $ comes from the following definition of  dot-product;
we define the max-plus dot-product between functions $V$ and $z$ from $\S$ to $\rb$ (and more generally for all functions defined on $\W$ or $\Z$)
as $
\langle z |   V \rangle = \bigoplus_{s \in \S} z(s) \otimes V(s)
= \max_{s \in \S} z(s) + V(s)$. We then have for any $\beta \in \rb^\Z$,
$\langle  {\Z\!}^\top V | \beta \rangle = \langle V | \Z \beta \rangle$, hence the transpose notation.  
We can also define the residuation as:
\BEQ 
\textstyle {{\Z\!}^{\top+}}\beta(s) = \min_{z \in \Z} \beta(z)  - z(s),
\EEQ
 so that ${{\Z\!}^{\top+}}{{\Z\!}^\top}$ goes from $\rb^\S$ to $\rb^\S$. The operator $ {{\Z\!}^{\top+}}{{\Z\!}^\top}$ on functions from $\S$ to $\rb$ is also the projection  on the image of ${{\Z\!}^{\top+}}$;
moreover, $ {{\Z\!}^{\top+}}{{\Z\!}^\top}  V \geqslant V$ for all $V$ and $ {{\Z\!}^{\top+}}{{\Z\!}^\top}$ is idempotent and non-expansive for the $\ell_\infty$-norm.
Note that $   {{\Z\!}^{\top+}}\beta = - \Z ( -\beta)$ so that properties of $   {{\Z\!}^{\top+}}\beta$ can be inferred from the ones of $\Z$. Similarly $\Z^{\top }V  = -\Z^+ (-V)$ and,
$  {{\Z\!}^{\top+}}{{\Z\!}^\top}  V = - \Z \Z^+ ( - V)$.

\subsection{Reduced value iteration}

Extending~\cite{akian2008max} to MDPs, we consider of $V_{t+1}   =   TV_t$,  the iteration $  V_{t+1} =   \W \W^+  {{\Z\!}^{\top+}}{{\Z\!}^\top} T V_t.$
If $V_t$ is represented as $\W \alpha_t$, then $V_{t+1}$ is represented as
$\W \alpha_{t+1}$ where $\alpha_{t+1} = \W^+  {{\Z\!}^{\top+}}{{\Z\!}^\top} T \W \alpha_t$, which we can decompose as
$
\beta_{t+1} = {{\Z\!}^\top} T \W \alpha_t$  and  $\alpha_{t+1} = \W^+  {{\Z\!}^{\top+}}\beta_{t+1} =  ({{\Z\!}^\top} \W)^+   \beta_{t+1}$.

The key point is then that the two operators $ {{\Z\!}^\top} T \W \!:\! \rb^\W \to \rb^\Z$ and  $\W^+  {{\Z\!}^{\top+}} \!:\! \rb^\Z \to \rb^\W$, can be pre-computed at a cost that will be independent of the discount factor $\gamma$. Indeed, given
$\langle z | w \rangle =  \max_{s \in \S} z(s)  + w(s)$ and
$\langle z | T w \rangle = \max_{s \in \S} z(s)  + Tw(s)$, we have:
\BEAS
 {{\Z\!}^\top} T \W \alpha(z) \!\!
 & \!\!\!=\!\! \!&  \!\max_{s \in \S} T \W \alpha (s) + z(s)
 = \max_{s \in \S} \max_{w \in \W} \gamma \alpha(w) + Tw(s) + z(s)
 = \max_{w \in \W} \gamma \alpha(w) + \langle z | T w \rangle \\
 \!\!     \W^+  {{\Z\!}^{\top+}}\beta(w) \!\!& \!\!\!= \!\!\!& 
\!\min_{s \in \S}  {{\Z\!}^{\top+}}\beta(s) - w(s)
 =  \min_{s \in \S}  \min_{z \in \Z} \beta(z) - z(s)  - w(s)
 = \min_{z \in \Z} \beta(z)  -\langle z | w \rangle.
 \EEAS
 Therefore, the computational complexity of the iteration is $O(|\W|\cdot |\Z|)$, once the $|\W| \cdot |\Z|$ values $
\langle z | w \rangle=  \textstyle\max_{s \in \S} z(s)  + w(s)$,
 and  $
\langle z | T w \rangle= \textstyle\max_{s \in \S} z(s)  + Tw(s)$    have been computed.  This requires some compilation time which is \emph{independent} of $\gamma$ and the final required precision. Note that if these values are computed  up to some precision $\varepsilon$, then we get an overall extra approximation factor of  $\varepsilon /  (1 - \gamma)$.
The iterations then become
\BEA
\label{eq:beta} \beta_{t+1}(z) & = &   {{\Z\!}^\top} T \W \alpha(z) = \textstyle \max_{w \in \W} \gamma \alpha_t(w) + \langle z | T w \rangle \\
\label{eq:alpha} \alpha_{t+1}(w) & = & \W^+  {{\Z\!}^{\top+}}\beta_{t+1}(w) =  \textstyle\min_{z \in \Z} \beta_{t+1}(z)  - \langle z | w \rangle.
\EEA
As seen below, they correspond to $\gamma$-contractant operators and are thus converging exponentially fast. In the worst case (full graph), the complexity is $O(|\W| \cdot |\Z|)$. Moreover, as presented below, a good approximation of $V_\ast$ by $\W$ and $\Z^\top$ leads to an approximation guarantee (see proof in  App.~\ref{app:propapprox}).

\begin{proposition}
\label{prop:approx}
(a) The operator $\hat{T} = \W \W^+  {{\Z\!}^{\top+}}{{\Z\!}^\top} T $ is $\gamma$-contractive and has a unique fixed point~$V_\infty$. (b) If 
$
\|  \W \W^+   V_\ast - V_\ast \|_\infty  \leqslant \eta
$ and $
\|   {{\Z\!}^{\top+}}{{\Z\!}^\top}  V_\ast - V_\ast \|_\infty \leqslant \eta
$, then $\| V_\infty - V_\ast \|_\infty \leqslant \frac{ 2\eta}{1-\gamma}$.
 \end{proposition}
Therefore, an approximation guarantee for $V_\ast$ translates to an approximation error multiplied by the horizon $\tau= (1-\gamma)^{-1}$. Thus large horizons $\tau$ (i.e., large $\gamma$) will degrade performance (see examples in \myfig{fixed_basis_mdp}). If we   consider a discount factor of $\gamma^\rho$ (corresponding to the operator $T^\rho$ instead of $T$, for $\rho>1$, see below), in the result above, $\tau $ is replaced by $(1+\tau/\rho)$ (see proof in  App.~\ref{app:propapprox}).

\paragraph{Algorithmic complexity.} After $t$ steps of the approximate algorithms in \eq{beta} and \eq{alpha}, starting from zero, we get $V_t$
such that $\|\hat{T} V_t - V_t \|_\infty \leqslant \gamma^t \| \hat{T}V_0 - V_0 \|_\infty
\leqslant \gamma^t {\rm range}(r)$, and thus such that
$\| V_t - V_\infty \|_\infty \leqslant  {\gamma^t {\rm range}(r)}{(1-\gamma)^{-1}}$, with the overall bound
$\| V_t - V_\ast \|_\infty \leqslant ({2 \eta + \gamma^t {\rm range}(r)}){(1-\gamma)^{-1}} $ if the assumption~(b) in Prop.~\ref{prop:approx} is satisfied. Thus to get an error of $  {   \varepsilon {\rm range}(r) } \tau$ for a fixed $\varepsilon$,  it is sufficient to have
an approximation error $\eta =  {\rm range}(r)  {\varepsilon}/{4} $ and a number of iterations $t =  { \log( 2/\varepsilon)}{(1-\gamma)^{-1}}= \log( 2/\varepsilon) \tau $. Thus, the overall complexity will be proportional to
$ |\W| \cdot |\Z| \cdot  { \log( 2/\varepsilon)}\cdot \tau $ in addition to the compilation time required to compute the $|\W| \cdot |\Z|$ values $\langle z|w \rangle$ and $\langle z|T w \rangle$ (which is independent of $\gamma$).

\paragraph{Extensions.}
We  can apply the reasoning above to $T^\rho$ and replace $\gamma$ by $\gamma^\rho$, with then $\rho$ fewer iterations in the leading terms, but a more expensive compilation time.  The horizon $\tau$ is then  equal to $\tau_\rho = (1 - \gamma^\rho)^{-1}$ which is equivalent to $\tau/\rho$ for $\rho = o(\tau)$.  Moreover,  when $\rho = O(\tau)$, the approximation error is reduced and we only need to have $\eta = \rho \cdot  {\rm range}(r)   {\varepsilon}/{4} $. This will also be helpul within matching pursuit (see \mysec{greedy}).

\section{Approximation by Max-plus Operators}
\label{sec:approx}

In this section, we consider several classes of functions $\W$ and $\Z$, with their associated approximation properties, that are needed for Prop.~\ref{prop:approx} to provide interesting guarantees. 
  Some of these sets are already present in~\cite{akian2008max} (distance and squared distance functions), whereas others are new (piecewise constant approximations and Bregman divergences). We present examples of such approximations in
\myfig{fixed_basis} and \myfig{fixed_basis_mdp}, where we note a difference between the  approximation of some function $V$ by $\W \W^+ $ or ${{\Z\!}^{\top+}}{{\Z\!}^\top}  V$ and their approximation with the same basis functions \emph{within an MDP}, as obtained from Prop.~\ref{prop:approx} (with $\rho$  very large, it would be equivalent, but not for small~$\rho$).

\subsection{Indicator functions and piecewise constant approximations}
\label{sec:clusters}
\label{sec:constant}

We consider functions $w:\S \to \rb$ of the form $w(s) = 0$ if $s \in A(w)$ and $-\infty$ otherwise, where $A(w)$ is a subset of $\S$, with $(A(w))_{w \in \W}$ forming a partition of $\S$. For simplicity, we assume here that $\Z = \W$.
The image of $\W$ is the set of piecewise constant functions with respect to this partition. Given a function $V$, $\W\W^+ V$ is the lower approximation of $V$ by a piecewise constant function, while $ {{\Z\!}^{\top+}}{{\Z\!}^\top}   V$ is the upper approximation of $V$ by a piecewise constant function (see \myfig{fixed_basis}).

\paragraph{Reduced iteration.}
Since the image of $ {{\Z\!}^{\top+}}$ and $\W$ are the same, the iteration reduces to
$ V_t =\W\W^+    {{\Z\!}^{\top+}}{{\Z\!}^\top}  TV_{t-1} = {{\Z\!}^{\top+}}{{\Z\!}^\top}  TV_{t-1} $. Thus, representing 
 $V_t$ as $ {{\Z\!}^{\top+}} \alpha_t$
 with $\alpha_t \in \rb^{\Z}$, we get  
\BEQ
\textstyle
 \alpha_{t+1}(w) = \max_{w' \in \W, (w,w') \in \E(A) }  R(w,w') +  \gamma \alpha_t(w') , 
 \EEQ
where $R(w,w') =  \max_{s \in A(w)} \max_{s' \in A(w')} \bar{r}(s,s')$ and $\E(A)$ is the set of $(w,w')$ such that $R(w,w')$ is finite.
This is  is exactly a deterministic MDPs on the clusters. The complexity of the algorithm depends on some \emph{compilation} time, in order to compute the reduced matrix $R$ (or the one corresponding to $T^\rho$), and some running-time \emph{per iteration}. These depends whether we consider a dense graph or a sparse $d$-dimensional graphs (corresponding to a $d$-dimensional grid). The complexities are presented below (with all constants dependent on $d$ removed), and proved in App.~\ref{app:comp}.

\vspace*{-.1cm}

\begin{center}
\begin{tabular}{|l|l|l|}
\hline
 & Compilation time & Iteration time   \\
 \hline
 Value iteration (dense) & $|\S|^3 \log \rho \textcolor{white}{\big|} $ & $|\S|^2$    \\
 Value iteration (sparse) & $|\S| \rho^{2d}$ & $ |\S| \rho^d$  \\
 Reduced MDP (dense) & $|\S|^2 + \min\{ |\S|^3 \log \rho, |\W| |\S|^2 \rho \} $ & $|\W|^2$  
  \\
 Reduced MDP (sparse) & $|\S| + \min \{ |\S| \rho^{2d} , |\W| |\S| \rho \}$ &  $|\W| \rho^d$
  \\
\hline
\end{tabular}
\end{center}

\vspace*{-.1cm}

For all cases above, the optimization error (to approximate $V_\infty$ in Prop.~\ref{prop:approx}) is ${\rm range}(r) \tau e^{- \rho t/\tau}$; thus, in the sparse case, the number of iterations to achieve precision $
{\rm range}(r) \tau \varepsilon
$ is proportional to  $ (\tau/\rho) \log (1/\varepsilon)$. Below, for simplicity, we are only considering trade-offs for the sparse graph cases. For plain value iteration, having $\rho>1$ larger than one could be beneficial only when $d=1$ (otherwise, the compilation time scales as $|\S|  \rho^{2d}$ and the iteration running time as $|\S|  \rho^{d-1}$, which are both increasing in $\rho$). When using the clustered representation, it seems that the situation is the same, but as shown later the approximation error comes into play and larger values of $\rho$ could be useful (both for running time with fixed dictionaries and for stability of matching pursuit).

 \paragraph{Approximation properties.} See the proof of the following proposition in App.~\ref{app:propconstant}, which provided approximation guarantees for Prop.~\ref{prop:approx}. We assume that $\W$ is composed of piecewise constant functions.
 
 \begin{proposition}
 \label{prop:constant}   If we denote by $\eta(n,\S)$ the smallest radius of a cover of $\S$ by balls of a given radius, by considering the  Voronoi partition associated with the $n$ ball centers, we get:
\BEQ
\label{eq:approxpartition}
\textstyle
 \min_{\alpha \in \rb^\W} \| V - \W \alpha\|_\infty \leqslant \| V - \W\W^+ V\|_\infty  
\leqslant {\rm Lip}_p(V) [2\eta(n,\S)]^p,
\EEQ
where $ {\rm Lip}_p(V)$ is the $p$-th order H\"older continuity constant of $V$ (i.e., so that for all $s,s' \in \S$, $|V(s)-V(s')| \leqslant 
 {\rm Lip}_p(V) d(s,s')^p$).
\end{proposition}

\begin{figure}[t]
\begin{center}
\hspace*{-.5cm}
\includegraphics[width=3.75cm]{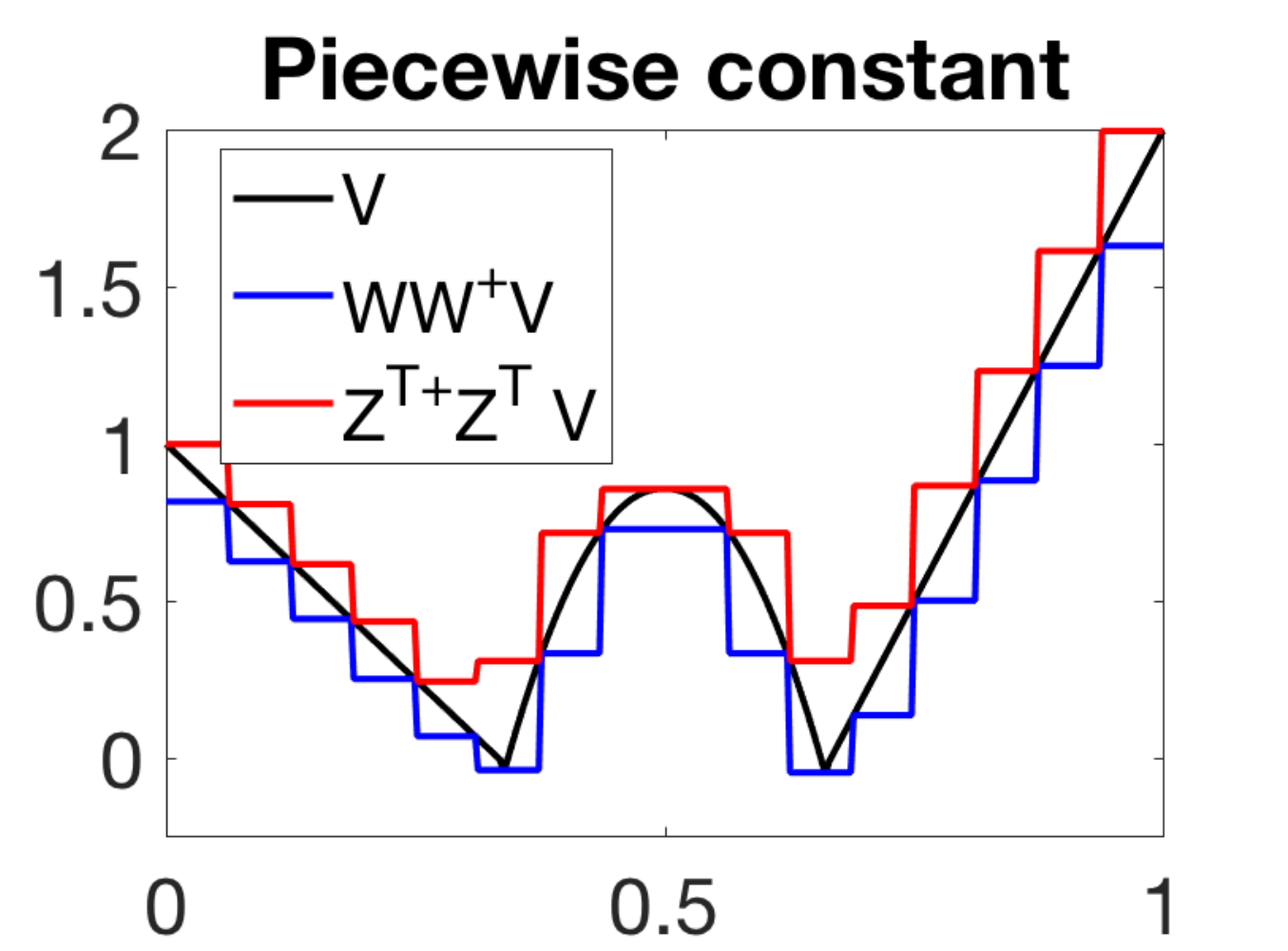} \hspace*{.234cm}
\includegraphics[width=3.75cm]{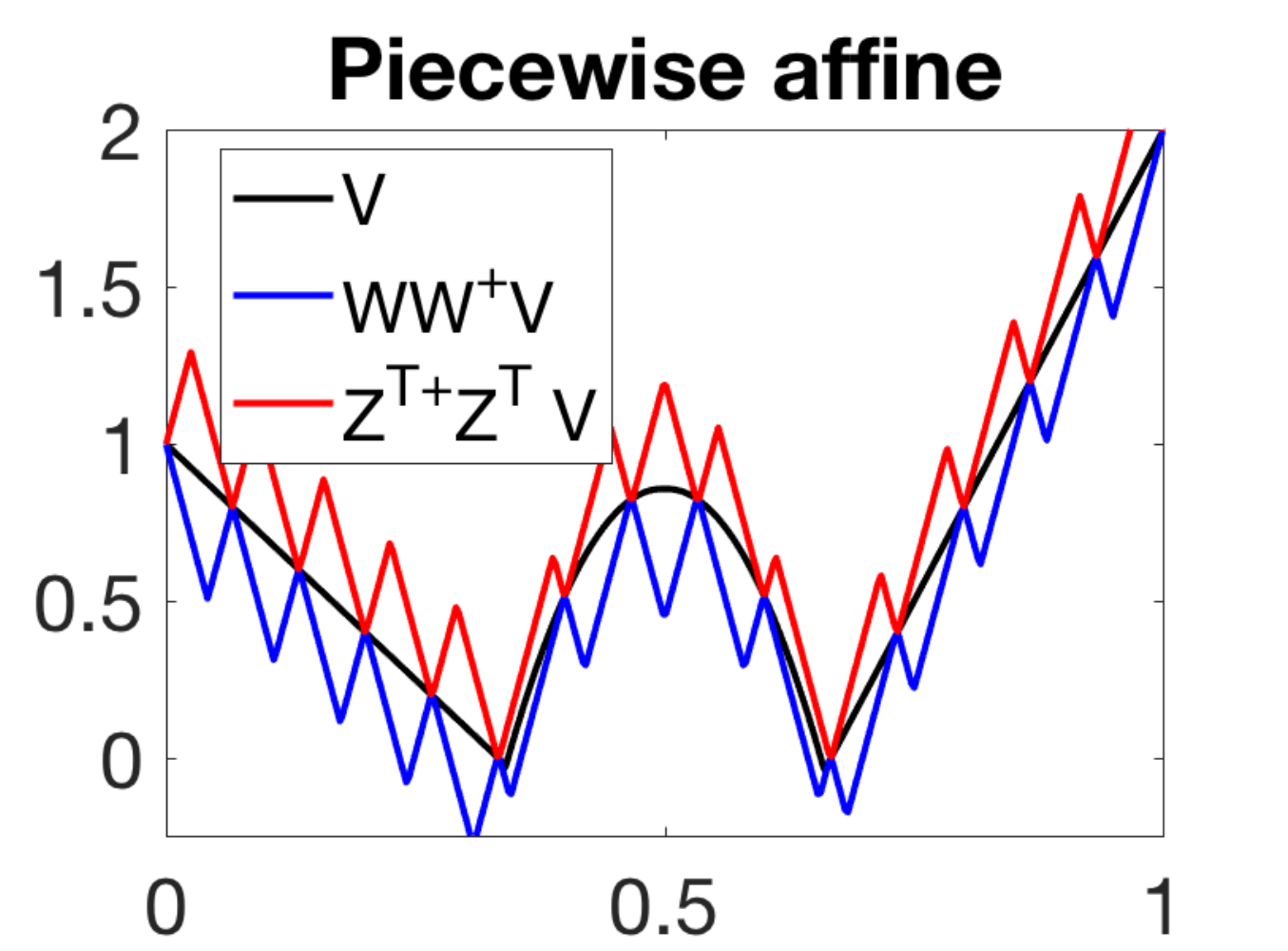} \hspace*{.234cm}
\includegraphics[width=3.75cm]{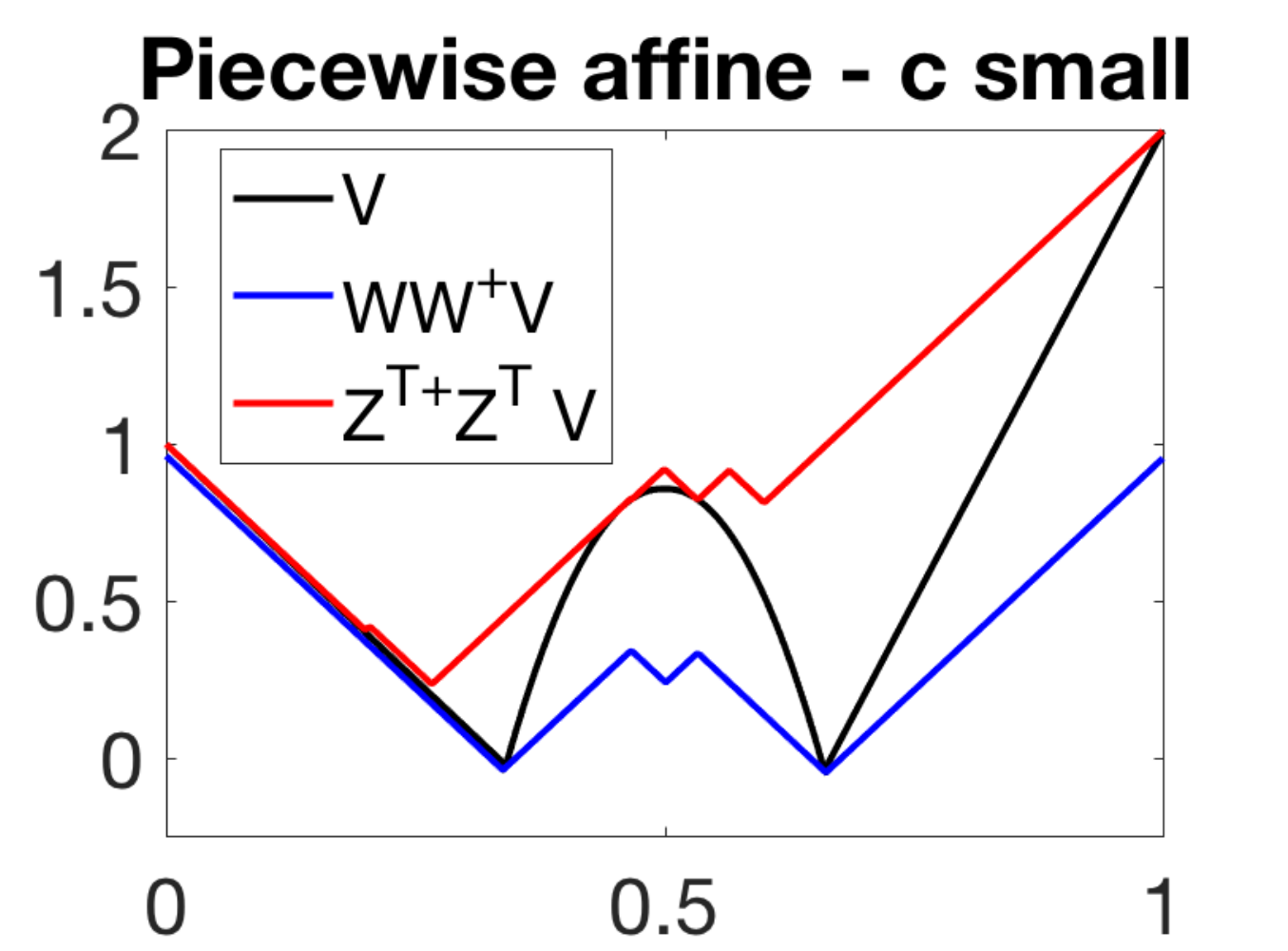} \hspace*{.234cm}
\includegraphics[width=3.75cm]{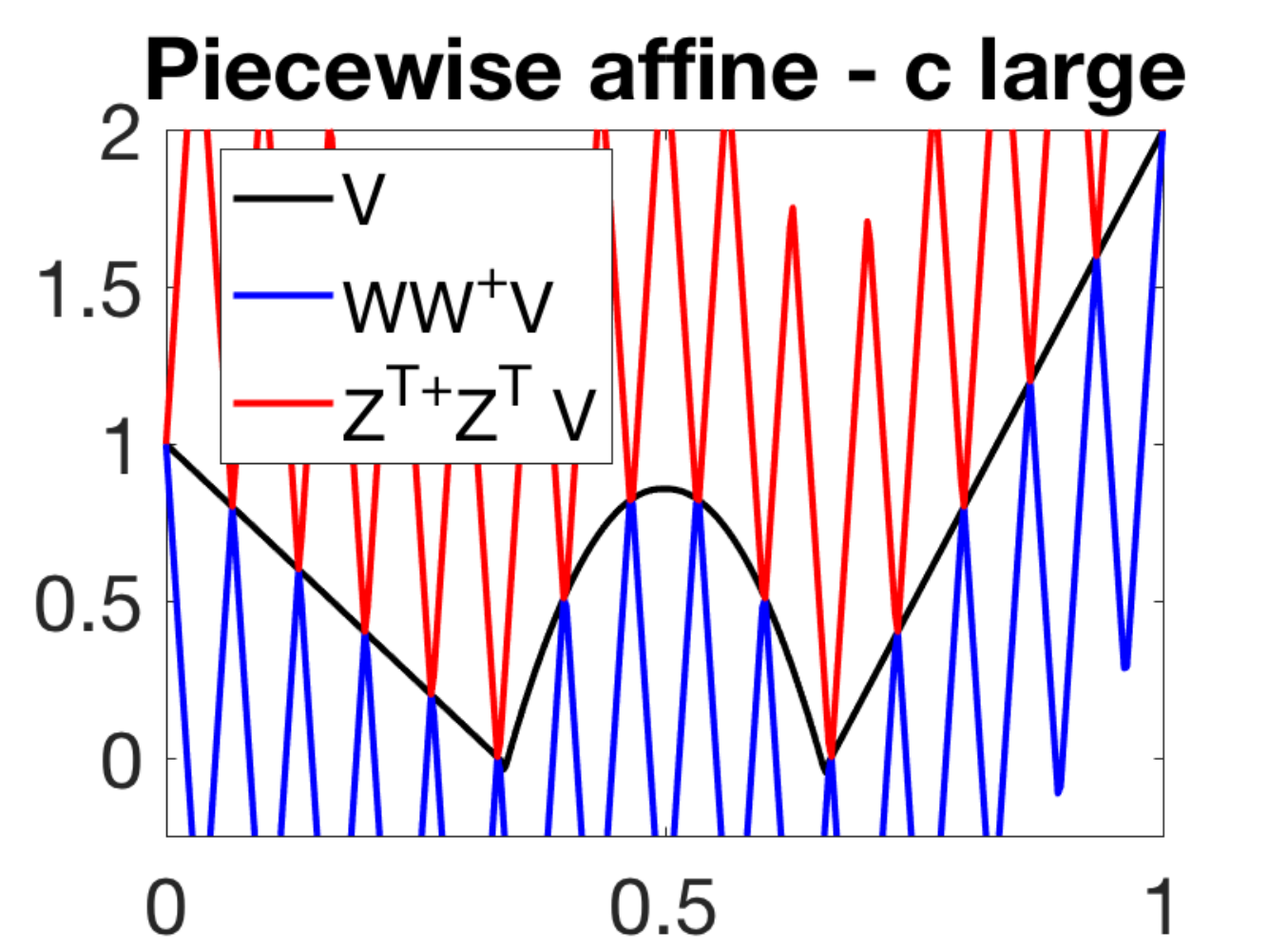}
\hspace*{-.5cm}\end{center}

\vspace*{-.4cm}

\caption{Approximation of a function $V$ with different finite basis with 16 elements. One-dimensional  case.  From left to right: piece-wise constant basis function, piecewise affine basis functions with well chosen value of~$c$, then too small and too large.
\label{fig:fixed_basis}}
\end{figure}

\begin{figure}[t]
\begin{center}
\hspace*{-.5cm}
\includegraphics[width=3.75cm]{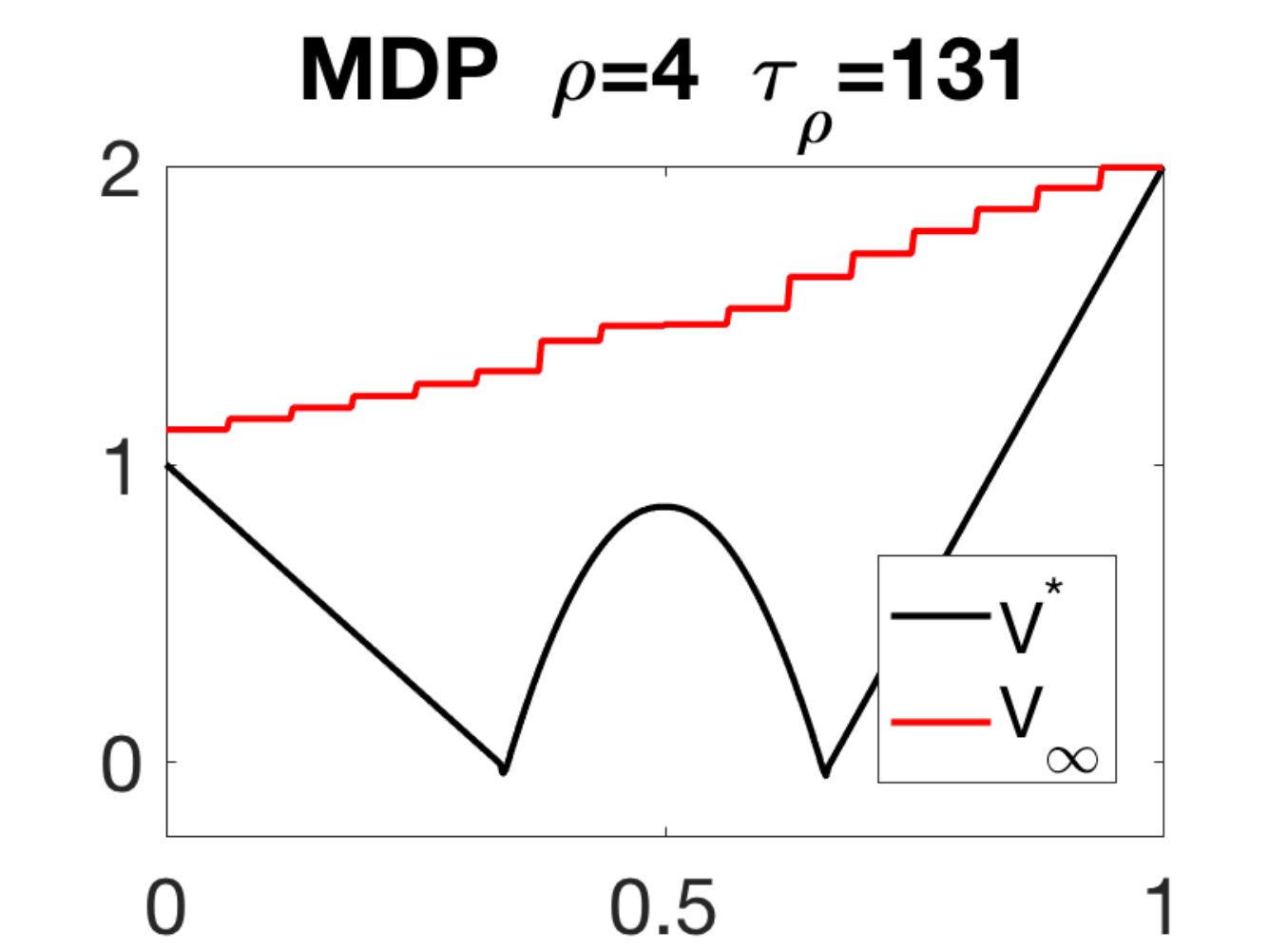} \hspace*{.234cm}
\includegraphics[width=3.75cm]{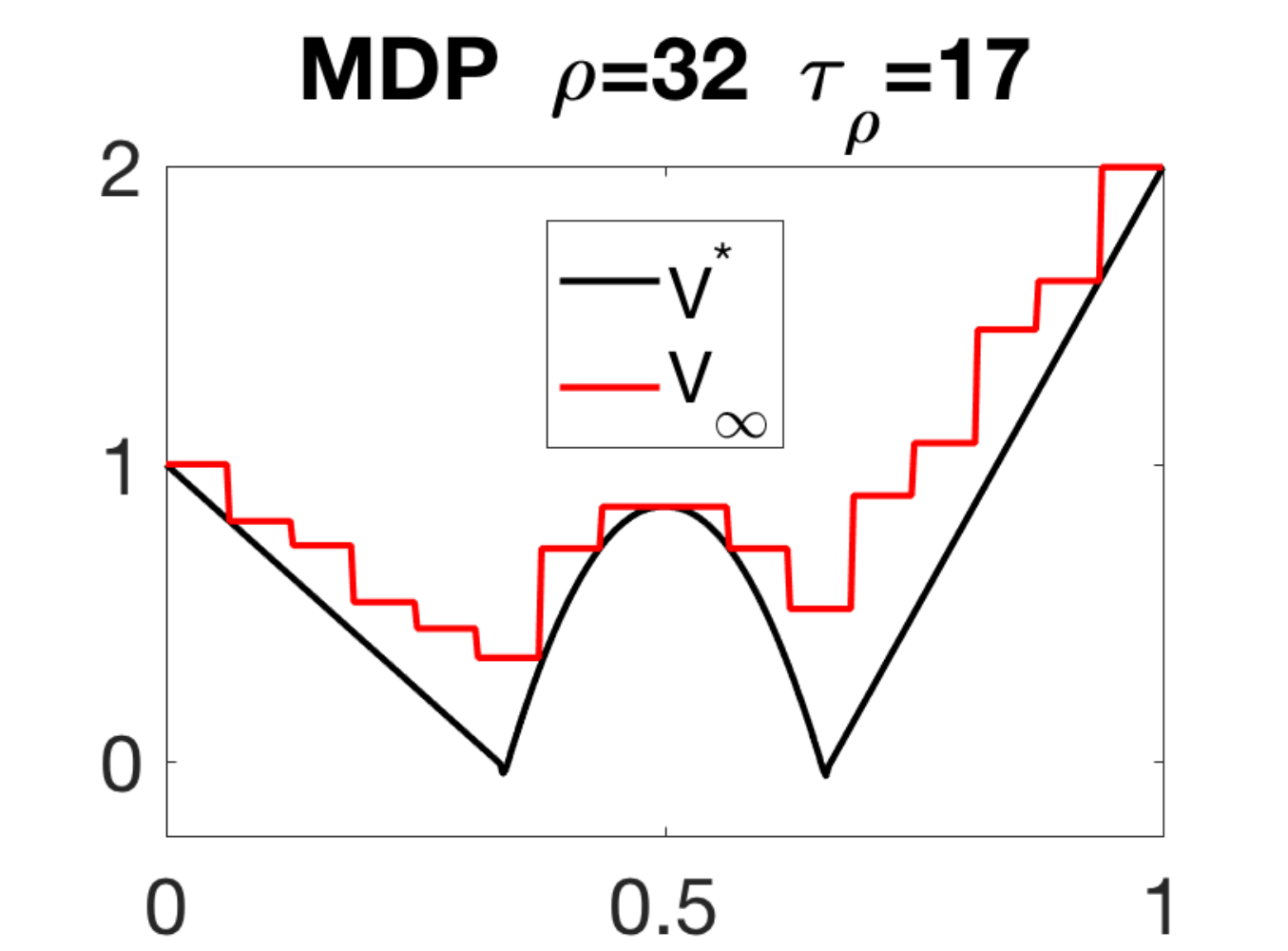} \hspace*{.234cm}
\includegraphics[width=3.75cm]{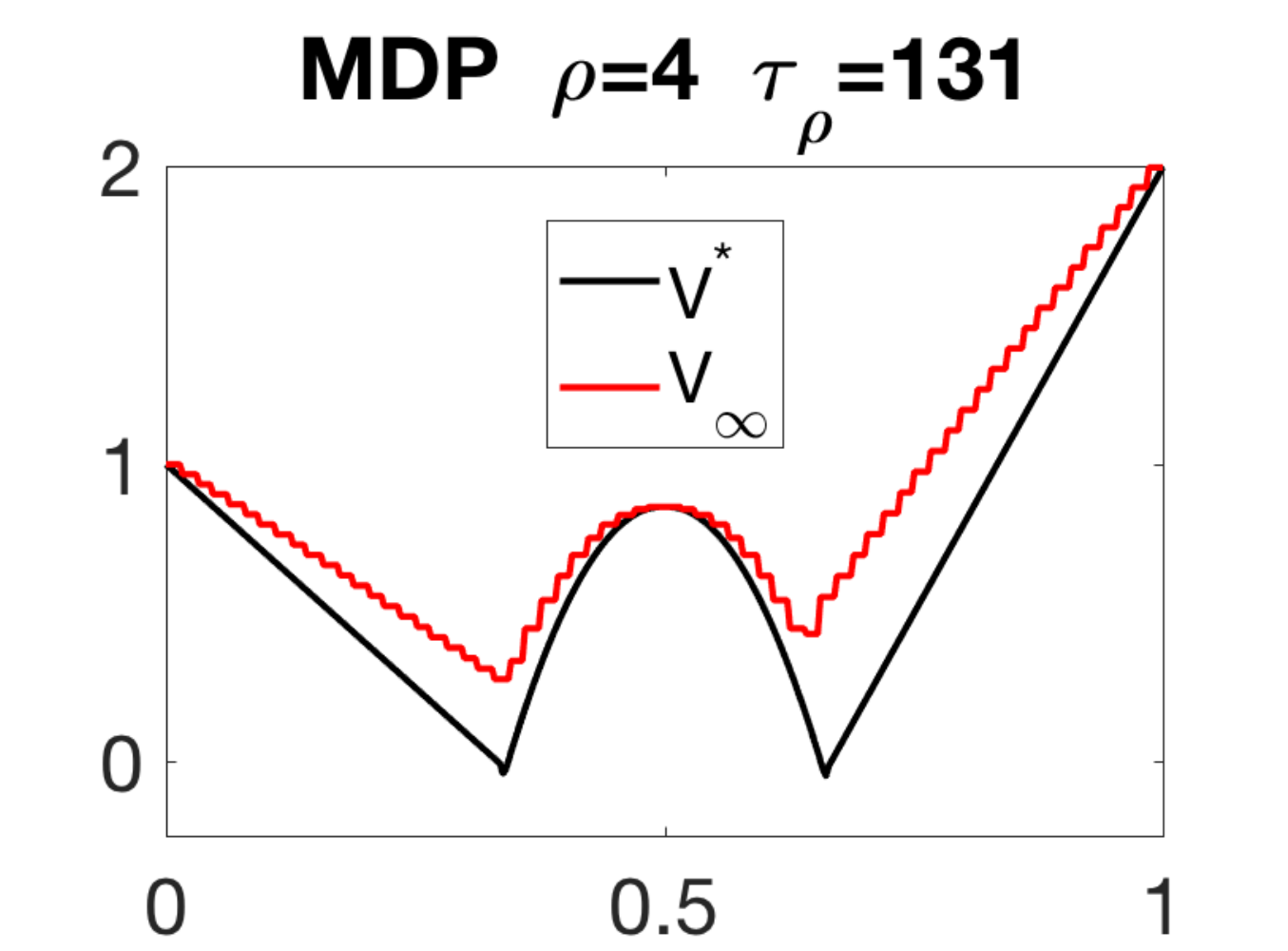} \hspace*{.234cm}
\includegraphics[width=3.75cm]{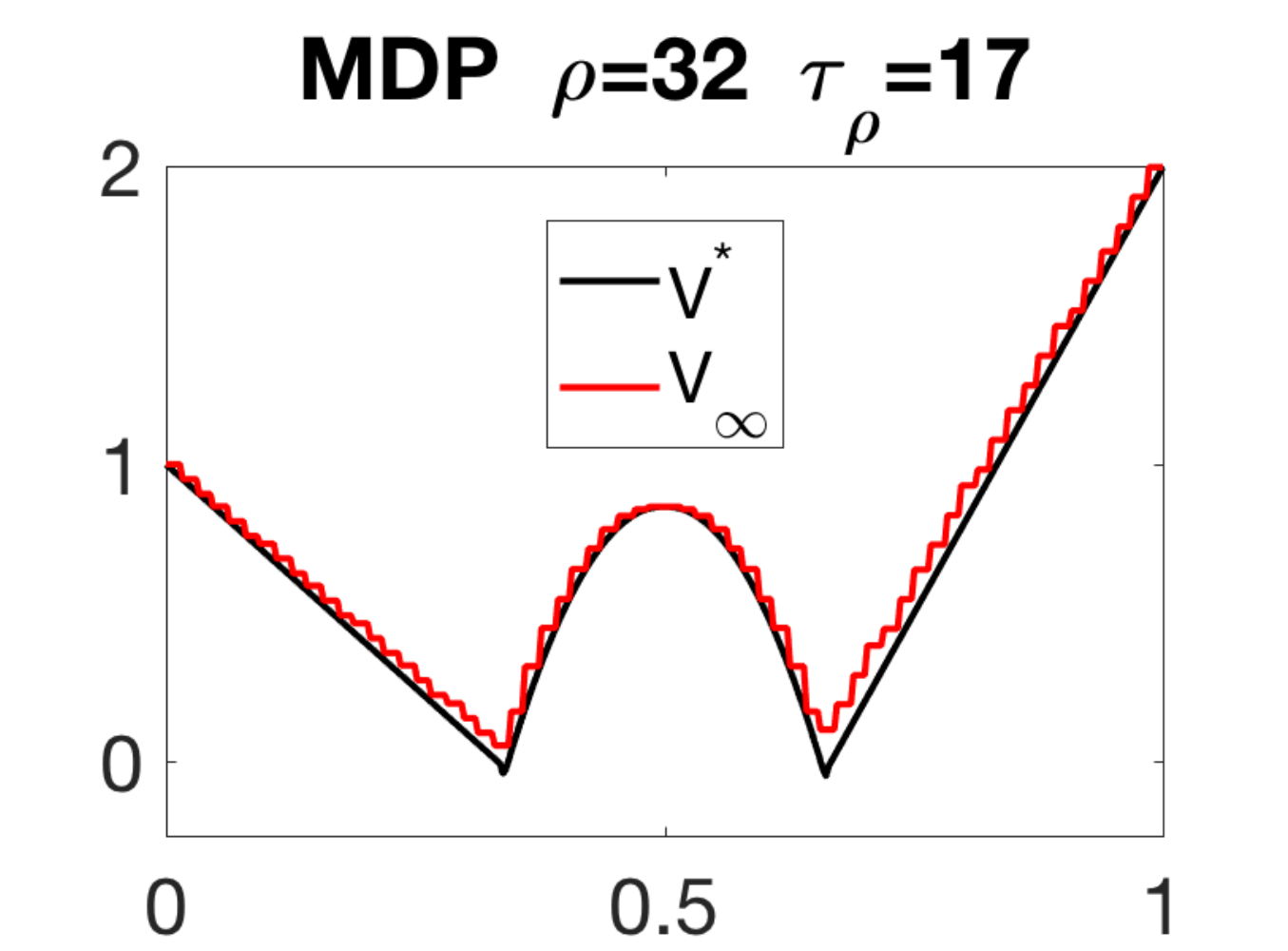}
\hspace*{-.5cm}\end{center}

\vspace*{-.4cm}

\caption{Approximation of a function $V$ with different finite basis with 16 or 64 elements, \emph{within the MDP} (i.e., leading to $V_\infty$ in Prop.~\ref{prop:approx}), and with values of $\rho$ that are $4$ and $32$. One-dimensional   case.  From left to right: $(n=16, \rho=4)$, $(n=16, \rho=32)$, $(n=64, \rho=4)$, $(n=16, \rho=32)$. \label{fig:fixed_basis_mdp}}
\end{figure}

While for factored spaces and with no assumptions on the optimal value function $V_\ast$ beyond continuity, $\eta(n,\S)$ may  not scale well, the approximation could be much better in special cases (e.g., when $V_\ast$ depends only on a subset of variables).
For factored spaces $\S = \S_1 \times \cdots \times \S_d$, with a $\ell_\infty$-metric, then $\eta_n(\S) \leqslant \max_{i \in \{1,\dots,d\}} \eta(n^{1/d}, \S_i)$.
If each factor $\S_i$ is simple (e.g., a chain graph), then $\eta(n, \S_i) = O(1/n)$, and we get 
$\eta_n(\S) = O(1/n^{1/d})$. Here, we do not escape the curse of dimensionality.

\paragraph{Going beyond exponential complexity.} In order to avoid the rate $O(1/n^{1/d})$ above, we can consider further assumptions:  if $V_\ast$ is well approximated by a piecewise constant function with respect to a specific partition (dedicated to $V_\ast$), the bound in \eq{approxpartition} can be greatly reduced. Note that the approximation with a small number of basis functions here is the same for $\W \W^\top$ and ${{\Z\!}^\top} ^+{{\Z\!}^\top}$ when $\Z = \W$ (this will not be the case in \mysec{affine} below).

We can get a reduced set of basis functions, if for example $V_\ast$ depends only on $k$ variables, then there is a partition for which the error in
\eq{approxpartition} is of the order $O(1/n^{1/k})$, and we can then escape the curse of dimensionality if $k$ is small and we can find these $k$ relevant variables. This will be done algorithmically in \mysec{greedy}, by considering a greedy algorithm in $[0,1]^d$ with sets $\prod_{i=1}^d \big[
\frac{j_i}{2^{k_i}}, \frac{j_i+1}{2^{k_i}}
\big]$ and a split according to a single dimension at every iteration. This variable selection does not need to be global and the covering number can benefit from local independences.

\paragraph{Optimal choices of hyperparameters.} The approximation error from Prop.~\ref{prop:constant} is of  order $(1\!+\! \tau/\rho) {\rm Lip}(V_\ast)  |\W|^{-1/k}$, and thus, for $\rho = O(\tau)$, in order to achieve a final approximation error of ${\rm range}(r) \tau \varepsilon$, we need $|\W|$ to be of the order  of  $  [   {\rm Lip}(V_\ast) {\rm range}(r)^{-1} \varepsilon^{-1} \rho^{-1} ]^{k}$. Without compilation time this leads to a complexity proportional 
to $ \big[   {\rm Lip}(V_\ast) {\rm range}(r)^{-1} \varepsilon^{-1}  \big]^{k} \tau \rho^{d-k-1} \log(1/\varepsilon)$; thus, it is advantageous to have large values of $\rho$ when $k=d$ (full dependence). This has to be mitigated by the compilation time that is less than $|\S| \rho^{2d}$, which grows   with $\rho$ and $d$. We will see in \mysec{exp} that larger values are also better for a good selection of dictionary elements in matching pursuit.

\subsection{Distance functions and piecewise-affine functions}
\label{sec:distance}
\label{sec:affine}
Following~\cite{akian2008max}, we consider functions of the form $w(s) = - c \cdot d(s,w)$ for $w \in \S$, and $d$ a distance on~$\S$. Thus $\W$ may be identified to a subset of $\S$. When $\S$ is a subset of $\rb^d$ and with $\ell_1$ or $\ell_\infty$ metrics, we get piecewise affine functions (see \myfig{fixed_basis}). We then have (see proof in App.~\ref{app:propdist}):

\begin{proposition}
\label{prop:dist}
(a) If $\W =\S$ and $c \geqslant {\rm Lip}(V)$ (Lipschitz-constant of $V$), then $\W\W^+V = V$.
\\
(b) If $c \geqslant {\rm Lip}_1(V)$, then $\| V - \W\W^+ V\|_\infty \leqslant  2c \cdot \max_{s \in \S} \min_{w \in \W}  d(w,s)$.
\end{proposition}
We thus get an approximation guarantee for Prop.~\ref{prop:approx} from a good covering of $\S$. The approximation also applies to~$\Z^\top$. In terms of approximation guarantees, then we need to cover $\S$ with sufficiently many elements of~$\W$, and thus with $n$ points in~$\W$ we get an approximation of exactly the same order than for clustered partitions (but with the need to know the Lipschitz constant ${\rm Lip}_1(V)$ to set the extra parameter).

In terms of approximation by a finite basis, a problem here is that being well approximated by some $\W \alpha$, for $|\W|$ small, does not mean that one can be well approximated by $\Z \beta$, with $|\Z|$ small (it is in one dimension, not in higher dimensions).  Moreover, these functions are not local so computing $\langle z | w \rangle$ and $\langle z | Tw \rangle$ could be harder. Indeed, the compile time is $O( ( |\E| + |\Z| \cdot |\S|) \cdot |\W|)$ while each iteration of the reduced algorithm is $O(|\W| \cdot |\Z|)$.

\paragraph{Distance functions for variable selection.} To allow variable selection, we can consider a more general family of distance functions, namely, function of the form $w(s) =  -\max_{i \in A } d(s_i,w_i)$ or $w(s) =  -\sum_{i \in A } d(s_i,w_i)$ where $A \subset \{1,\dots,d\}$ for factored spaces. We consider the second option in our experiments in \mysec{exp}.

\subsection{Extensions}

\paragraph{Smooth functions.} As outlined by~\cite{akian2008max}, in the Euclidean continuous case, we can consider square distance funtions, which we generalize to Bregman divergences in Appendix~\ref{app:bregman}. Note that these will approximate smooth functions with more favorable approximation guarantees (but note that in MDPs obtained from the discretization of a continuous control problem, optimal value functions are  non-smooth in general~\cite{crandall1983viscosity}). Finally, other functions could be considered as well, such as linear functions restricted to a subset, or ridge functions of the form $ w(s) = \sigma (A^\top s +b)$.

 \paragraph{Random functions.} If we select a random set of points from a Poisson process with fixed intensity, then the maximal diameter of the associated random Voronoi partition has a known scaling~\cite{calka2014extreme}, similar to the covering number up to logarithmic terms. Thus, we can use distance functions from \mysec{distance} (where we can sample the constant $c$ from an exponential distribution) or clusters from \mysec{clusters}, or also squared distances like in \mysec{quadratic}.

 \section{Greedy Selection by Matching Pursuit}
 \label{sec:greedy}
The sets of functions proposed in \mysec{approx} still suffer from the curse of dimensionality, that is, the cardinalities $|\W|$ and $|\Z|$ should still scale (at most linearly) with $|\S|$, and thus exponentially in dimension if $|\S|$ is a factored state-space. We consider here greedily selecting new functions, to use dictionaries adapted to a given MDP, mimicking the similar approach of sparse decompositions in signal processing and unsupervised learning~\cite{mallat2008wavelet,mairal2014sparse}.

We assume that we are given $\W$ and $\Z$, and we want to test new sets $\W_{\rm new}$
and $\Z_{\rm new}$, which are  close to $\W$ and $\Z$ (typically one function $w$ split in two, that is, $w = w_1 \otimes w_2 = \max \{w_1,w_2\}$ or one additional function $w_{\rm new}$). Note that pruning~\cite{gaubert2011curse} could also be considered as well.

We assume that we have the optimal $\alpha \in \rb^\W$ and $\beta\in \rb^\Z$ (after convergence of the iterations defined in \eq{beta} and \eq{alpha}), such that 
$ \beta = {{\Z\!}^\top} T \W \alpha$ and $\alpha = \W^+  {{\Z\!}^{\top+}}\beta$. We denote by $V = 
\W \alpha$ and $U = {{\Z\!}^{\top+}}\beta =   {{\Z\!}^{\top+}} {{\Z\!}^\top} T V$. The criterion will be based on considering the best improvement of the new projections $\W_{\rm new} \W_{\rm new}^+$ and ${{\Z\!}}_{\rm new}^{\top+} {{\Z\!}^\top}_{\rm new}$ to the relevant function.

\paragraph{Criterion for $\W_{\rm new}$.}
Given the current difference between $U$ and its projection
$U - \W\W^+ U = U - V \geqslant 0$, the criterion to minimize is $\|
U - \W_{\rm new} \W_{\rm new}^+ U
\|$, for a given norm $\| \cdot\|$.
If  $\W_{\rm new} = \W \cup \{{w}_{\rm new}\}$, we have 
$\W_{\rm new} \W_{\rm new}^+U(s)
=   \max \big\{ w_{\rm new}(s) - \langle w_{\rm new} | -U \rangle, V(s) \big\}  $, 
and our criterion thus becomes $\|
U - \W_{\rm new} \W_{\rm new}^+ U
\|_\infty
= \max_{s \in \S} \min \big\{
U(s) - V(s), U(s) -  w_{\rm new}(s) + \langle w_{\rm new} | -U \rangle
\big\}$ for the $\ell_\infty$-norm.

\paragraph{Criterion for $\Z_{\rm new}$.}
Given the current difference between $TV$ and its projection
$ {{\Z\!}^{\top+}}{{\Z\!}^\top}   TV  - TV  = U - TV \geqslant 0$,
 the criterion to minimize is  $\|
{{\Z\!}}_{\rm new}^{\top+} {{\Z\!}^\top}_{\rm new}  TV - TV
\|_\infty$. If $\Z_{\rm new} = \Z \cup \{{z}_{\rm new}\}$, we have
${{\Z\!}}_{\rm new}^{\top+} {{\Z\!}^\top}_{\rm new}  TV(s)
=  \min \big\{-
 z_{\rm new}(s)  + \langle TV |z_{\rm new} \rangle,
U (s) \big\}
 $, and and our criterion becomes
$\|
{{\Z\!}}_{\rm new}^{\top+} {{\Z\!}^\top}_{\rm new}  TV - TV
\|_\infty
= \max_{s \in \S} \min \big\{
U(s) - TV(s), - TV(s) - z_{\rm new}(s)   + \langle TV |z_{\rm new} \rangle
\big\}$.

Like in regular matching pursuit for classical linear approximation, allowing full flexibility for $z_{\rm new}$ or $w_{\rm new}$ would lead to trivial choices, here   $z_{\rm new} = -TV$ and $w_{\rm new} = U$: in our MDP situation, we essentially end up with few changes compared to plain value iteration as the new atoms are just obtained by using our own approximation $U$ or applying the transition operator to another approximation (i.e., $TV$). Thus, within matching pursuit, in order to benefit from a reduction in global approximation error, ensuring the diversity of the dictionary of atoms which we select from is crucial. To go beyond selection from a finite set (and learn a few parameters), we propose a convex relaxation of estimating $z_{\rm new}$ or $w_{\rm new}$ in App.~\ref{app:cvxrelax}.

\paragraph{Special case of partitions.}
In this case , we have  $\W\W^+ {{\Z\!}^{\top +}}{{\Z\!}^\top} = {{\Z\!}^{\top +}}{{\Z\!}^\top}$, and thus we only need to consider $\|
{{\Z\!}}_{\rm new}^{\top+} {{\Z\!}^\top}_{\rm new}  TV - TV
\|_\infty$ above. A simplification there is to consider the current set $A(w)$ which contains $s$ attaining the $\ell_\infty$ bound above, and only try to split this cluster.

\paragraph{Related work.}
Variable selection within an MDP could be seen as a special case of factored MDPs~\cite{guestrin2003efficient} where the reward function depends on a subset of variables, it would be interesting to study theoretically more complex dependences.
Beyond factored MDPs, variable selection and more generally function approximation within MDPs has been considered in several works~(see, e.g.,~\cite{bertsekas1989adaptive,patrascu2002greedy,keller2006automatic,mahadevan2007proto,busoniu2011cross,liu2015feature}), but do not provide the theoretical analysis that we provide here or consider linear function approximations, while we consider max-plus combinations; \cite{munos2002variable} considers variable resolution within the discretization of an optimal control problem, but not for a generic Markov decision process.

\section{Experiments}
\label{sec:exp}

All experiments can be exactly reproduced with the Matlab code which can be obtained from the author's webpage.\footnote{{\small \url{www.di.ens.fr/~fbach/maxplus.zip}}}

\paragraph{Simulations on discretizations of control problems on $[0,1]$.}
We consider the discretization $\S = \big\{  ({i-1})({S-1})^{-1}, \ i \in \{1,\dots,S\} \big\}$ of $[0,1]$, which we can identify to $\{1,\dots,S\}$.
Given the set of actions $\mathcal{A} = \{-1,1\}$, following~\cite{munos2000study}, we consider the dynamical systems
$dx/dt = a$ (going left or right). The goal of the control problem is to maximize $\int_0^u \eta^t b(x(t)) dt + \eta^t B(x(u))$ where $u$ is the exit time of $x$ from $[0,1]$ (i.e., reaching the boundary). We then define the value function $V(x)$ as the supremum over controls $a(\cdot)$ starting from $x(0)=x$.

Then $V(x)$ satisfies the Hamilton-Jacobi-Bellman (HJB) equation~\cite{crandall1983viscosity,munos2000study}:
$V(x) \log \eta +  |V'(x)|  + b(x)= 0$, with the boundary condition $V(x) \geqslant B(x)$ for $x \in \{0,1\}$. With the discretization above, with a step $\delta = 1/(S\!-\!1)$, we have the absorbing states $1$ and $S$. We thus need to construct the reward from $s$ to $s+1$, and from $s$ to $s-1$, for any $s \in \{2,\dots,S-1\}$ equal to the reweighted function $\delta b(x)$ for the reached state, and a discount factor equal to $\gamma = \eta^\delta$. From states $1$ and $S$, no move can be made and the reward $r(S,S)$ is equal to $(1-\gamma)B(1)$ and $r(1,1) = (1-\gamma)B(0)$. When~$\delta$ goes to zero, then the MDP solution converges to the solution of the optimal control problem.

In our example in \myfig{fixed_basis} and \myfig{fixed_basis_mdp}, we consider pairs $(b,V)$ that satisfy the HJB equation. In \myfig{perf_1d} (top), for the same problem, we consider the performance of our greedy method (matching pursuit with an $\ell_1$-norm criterion) or fixed basis for piecewise constant functions from \mysec{constant} and piecewise affine functions from \mysec{affine}, when the horizon $\tau_\rho$ varies (from values of $\rho$) and the number of basis functions varies. We can see (a) the benefits of piecewise affine over piecewise constant  functions in the sets $\W$ and $\S$ (that is, better approximation properties with the same number of basis functions), (b) the beneficial effect of larger $\rho$ (i.e., using $T^\rho$ instead of $T$), in particular for greedy techniques where the selection of good atoms does require a larger $\rho$.

\begin{figure}[t]
\begin{center}
\hspace*{-2.5cm}
\includegraphics[width=3.8cm]{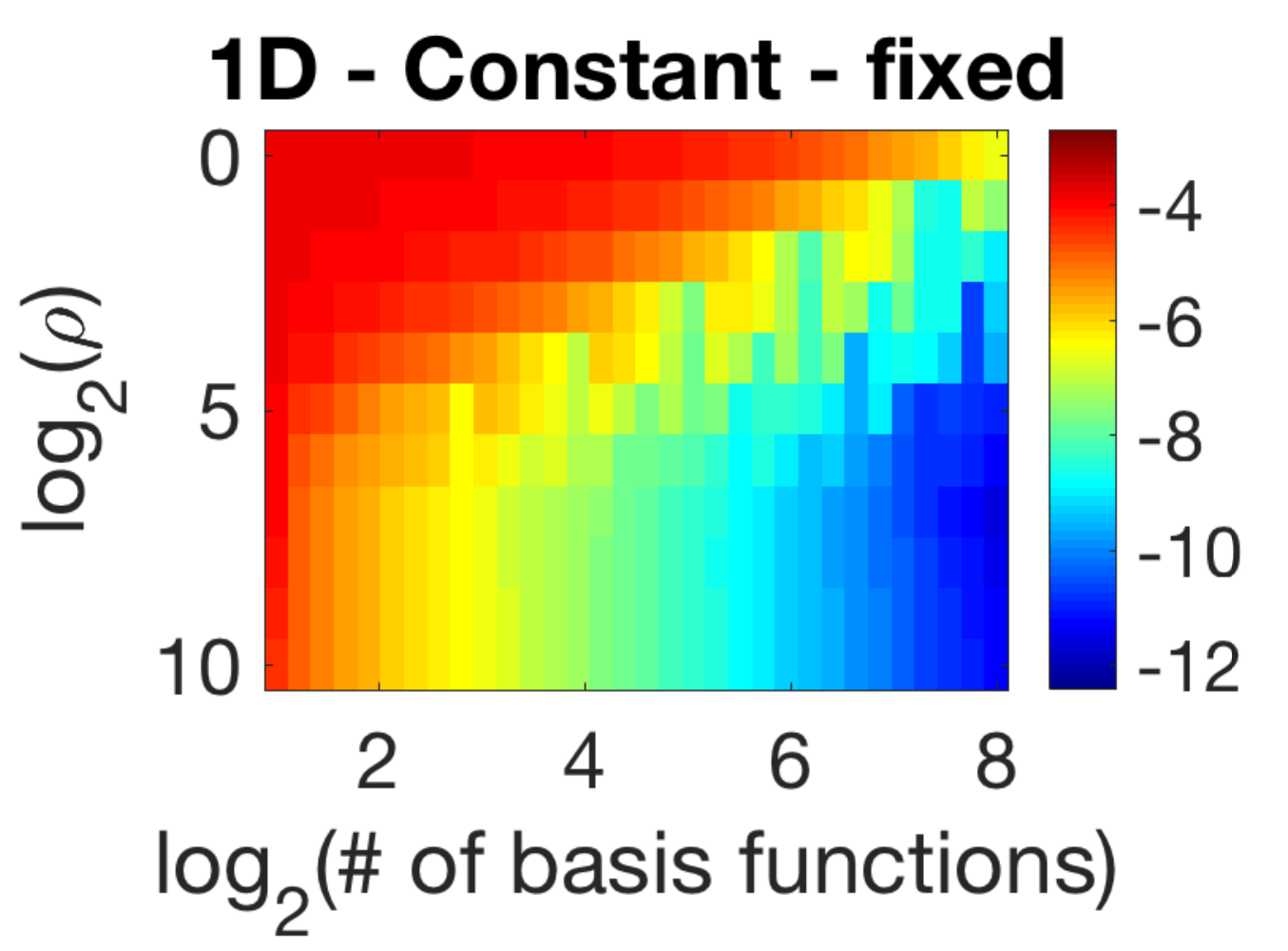} \hspace*{.1234cm}
\includegraphics[width=3.8cm]{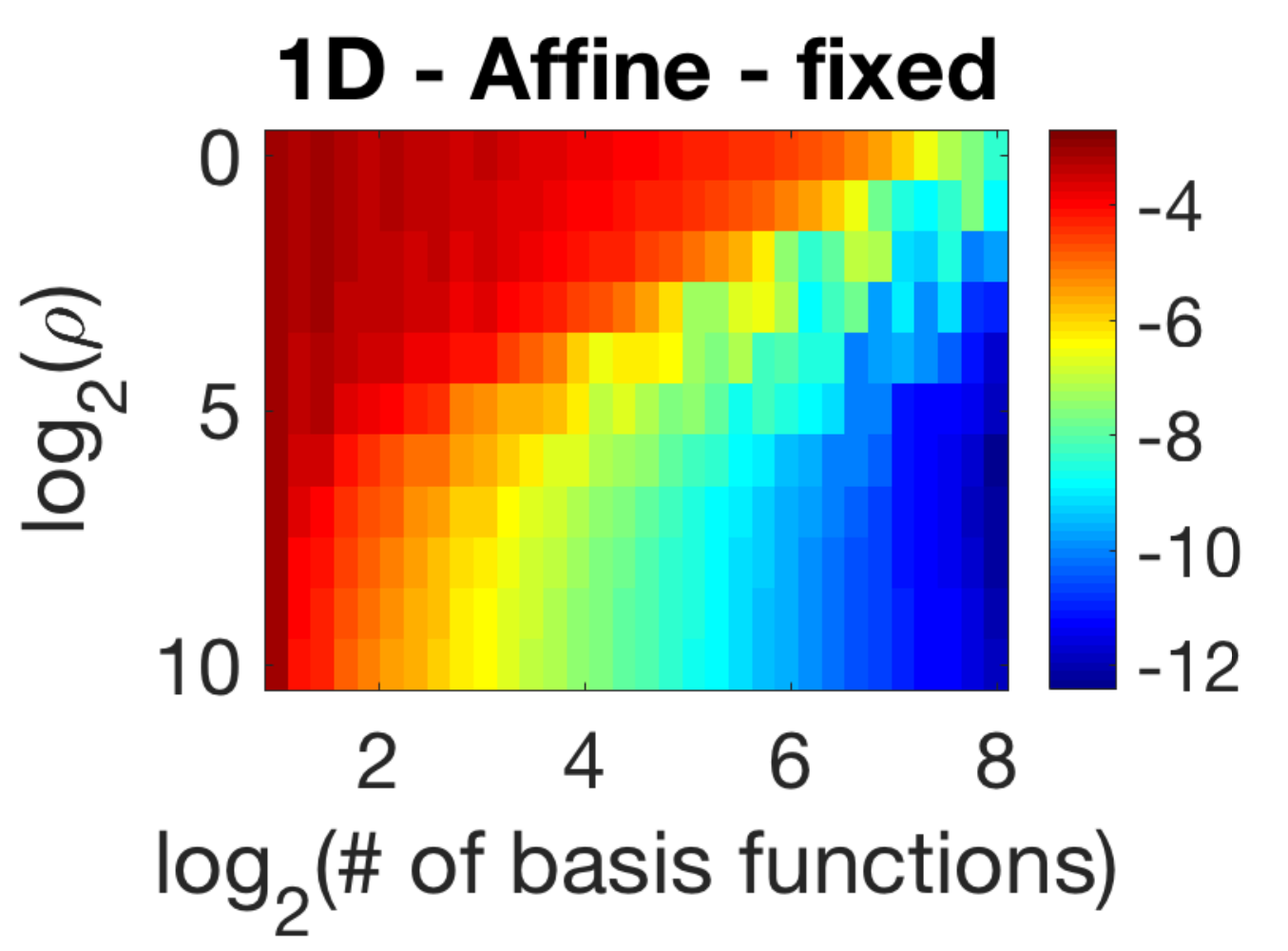} \hspace*{.1234cm}
\includegraphics[width=3.8cm]{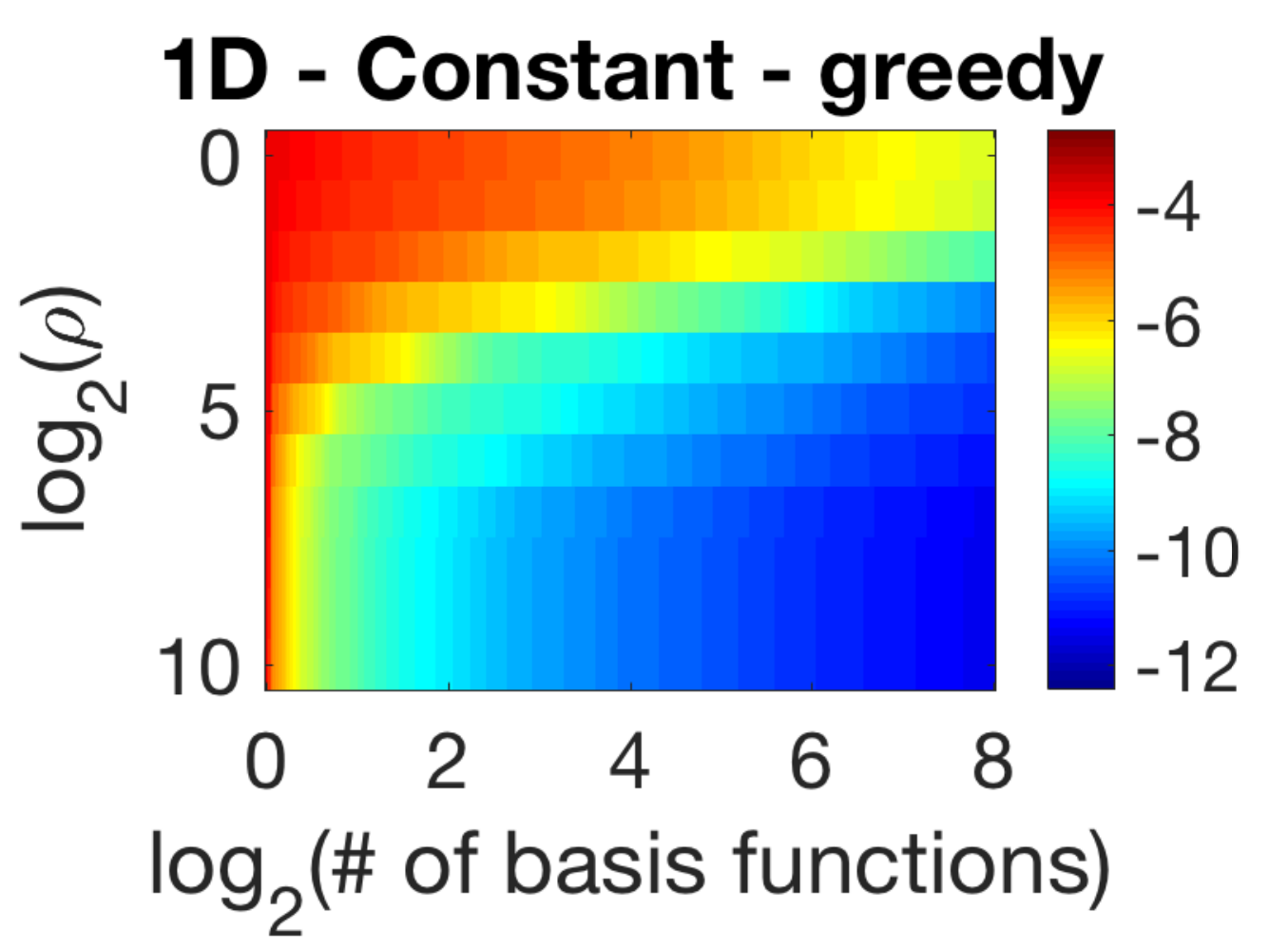} \hspace*{.1234cm}
\includegraphics[width=3.8cm]{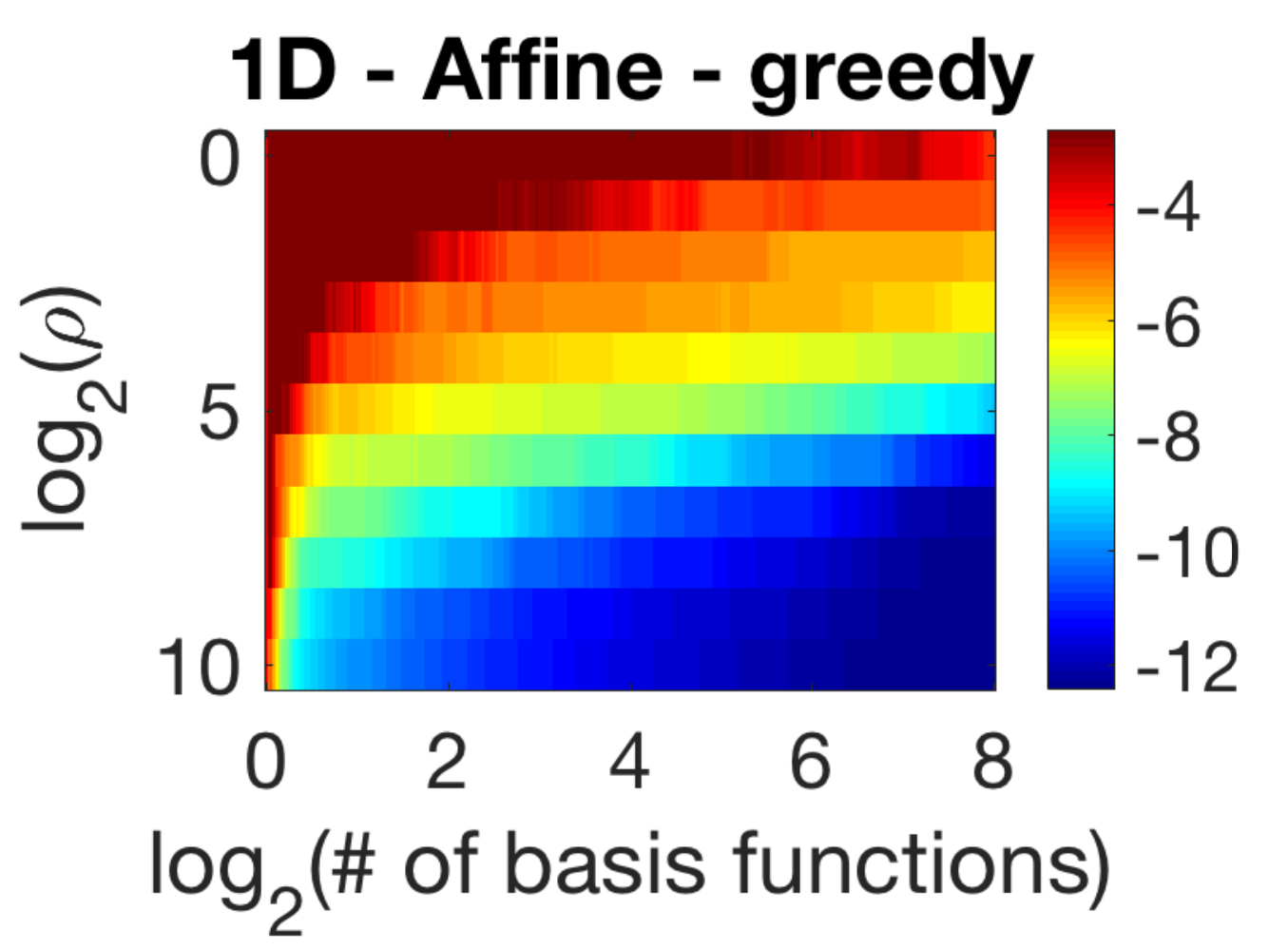}  
\hspace*{-2.5cm}

\vspace*{.2cm}

\hspace*{-2.5cm}
\includegraphics[width=3.8cm]{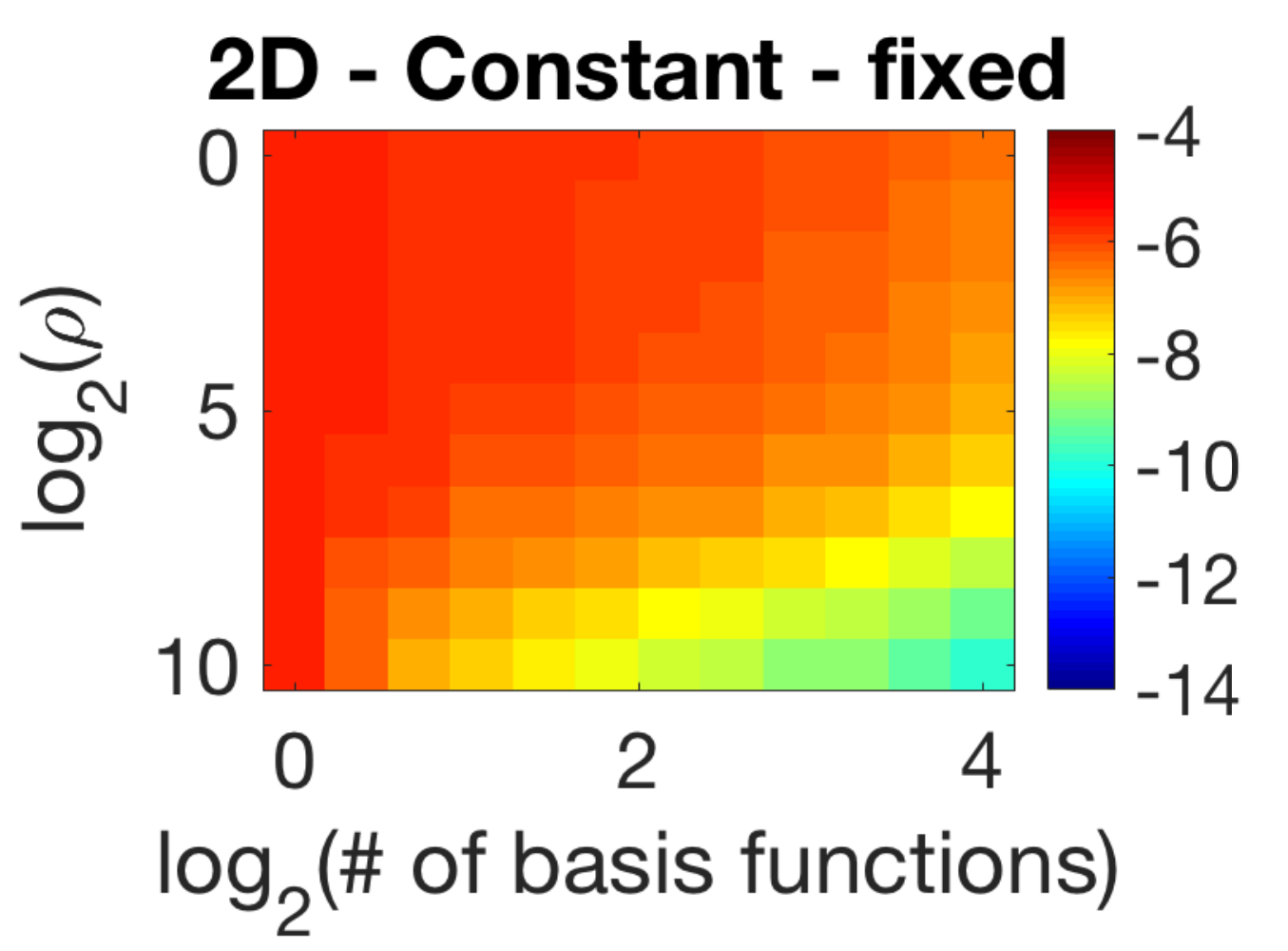} \hspace*{.1234cm}
\includegraphics[width=3.8cm]{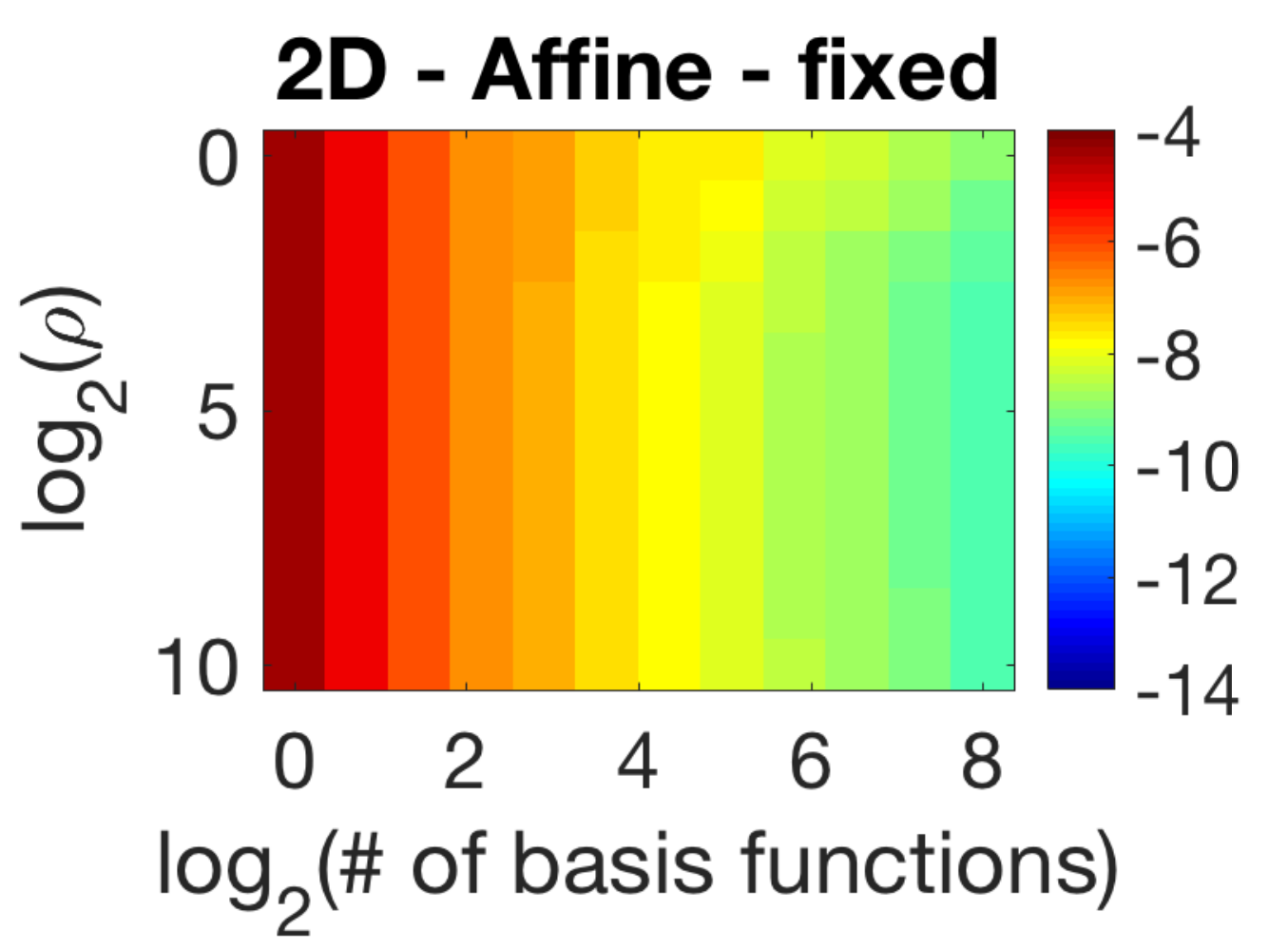} \hspace*{.1234cm}
\includegraphics[width=3.8cm]{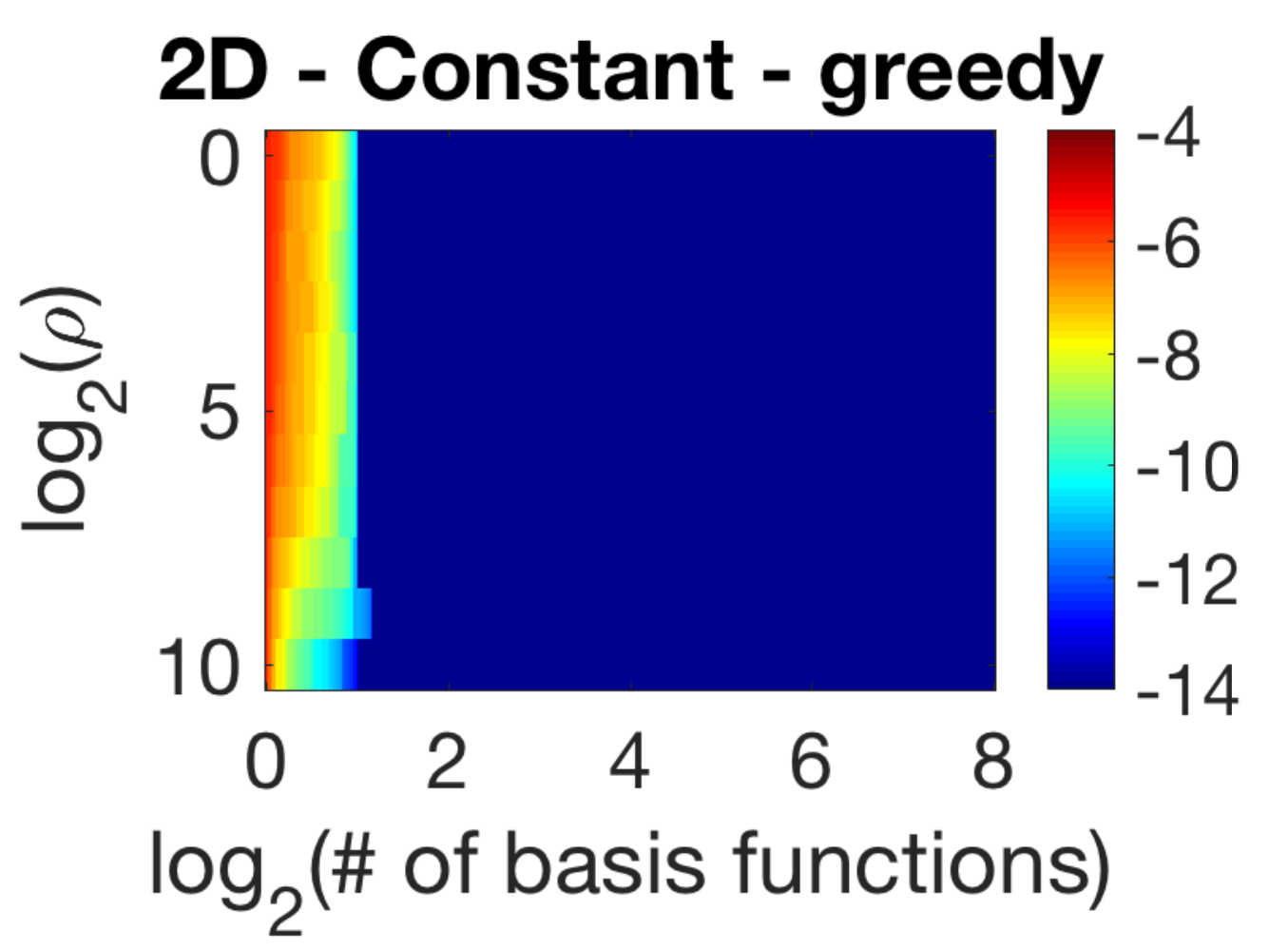} \hspace*{.1234cm}
\includegraphics[width=3.8cm]{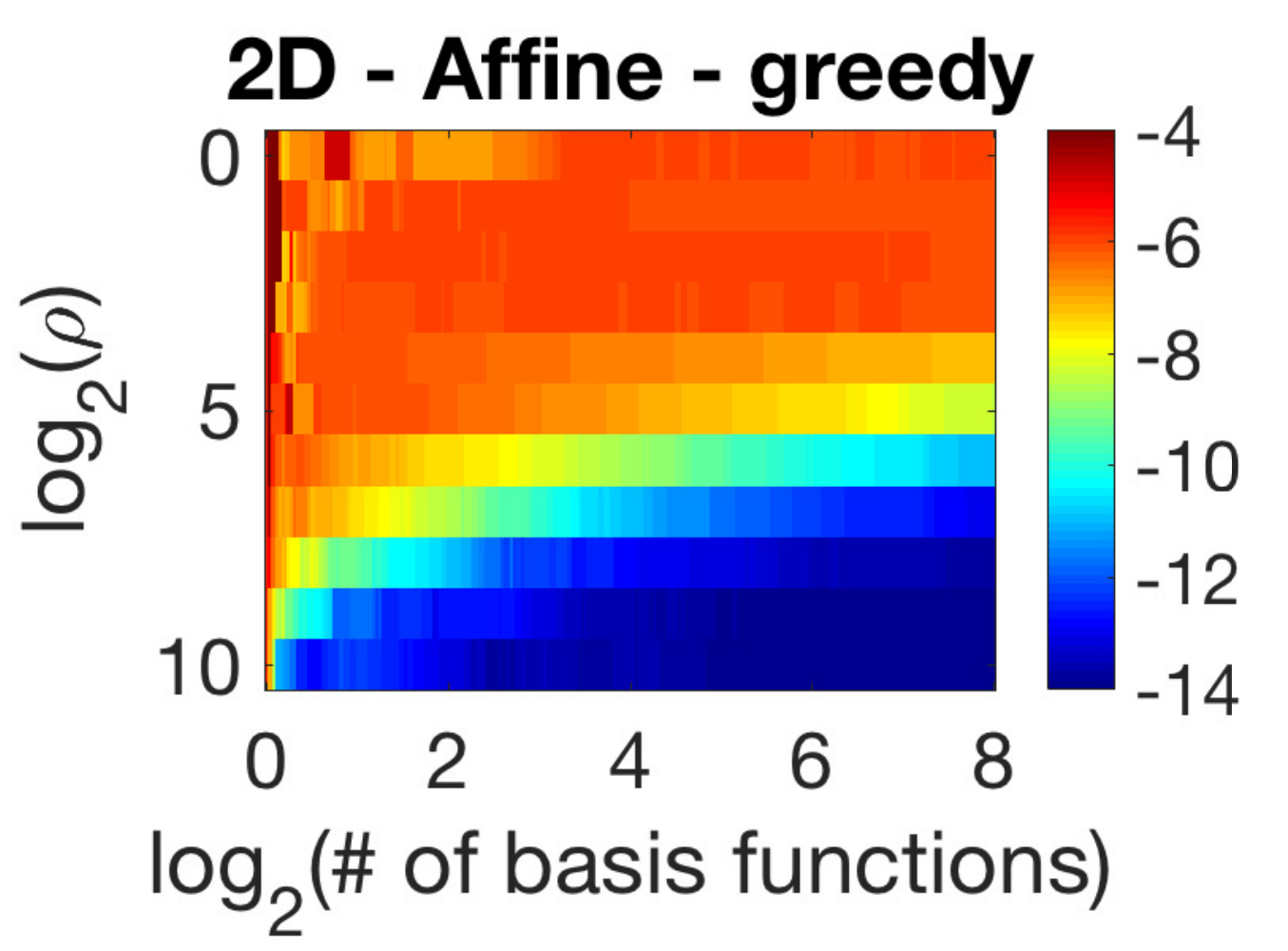}  
\hspace*{-2.5cm}

\end{center}

\vspace*{-.45cm}

\caption{Approximation error $\|V-V_\ast\|_1$ of a function $V$ as a function of $\rho$ and the number of basis functions, for piecewise constant or affine functions.  Top: One-dimensional case, bottom: two-dimensional case. Left: fixed non adaptive basis, right: greedy (matching pursuit). \label{fig:perf_1d}}
\end{figure}

\paragraph{Simulations on discretizations of control problems on $[0,1]^d$.} We consider two-dimensional extensions  where the optimal value function only depends on a single variable (see details and more experiments in Appendix~\ref{app:exp_2d}, with a full dependence on the two variables) and we show performance plots in \myfig{perf_1d} (bottom). Because of the sparsity assumption, the benefits of matching pursuit are greater than for the one-dimensional case (empirically, only relevant atoms are selected, and thus $\|V-V_\ast\|_1$ converges to zero faster as the number of basis functions increases).

\section{Conclusion}
 \label{sec:conclusion}
 
 In this paper, we have presented a max-plus framework for value function approximation with a greedy matching pursuit algorithm to select atoms from a large dictionary. While our current framework and experiments deal with low-dimensional deterministic MDPs, there are many avenues for further algorithmic and theoretical developments.
 
 First, for non-deterministic MDPs, the Bellman operator is unfortunately not max-plus additive; however,
  there are natural extensions to  general MDPs using measure-based formulations on probability measures on $\S$~\cite{hernandez2012discrete,lasserre2008nonlinear}, where using a policy $\pi(\cdot | \cdot)$, one goes from a measure~$\mu$ to $\mu'(s') = \sum_{s \in \S} \sum_{a \in \A} p(s'|s,a) \pi(a|s) \mu(s)$, with a reward going to $\mu$ and $\mu'$ equal to $ \sum_{s \in \S} \sum_{a \in \A} r(s,a) \pi(a|s) \mu(s)$. The difficulty here is the increased dimensionality of the problem due a new problem defined on probability measures. Also, going beyond model-based reinforcement learning could be done by estimation of model parameters; the max-plus formalism is naturally compatible with dealing with confidence intervals. Finally, it would be worth exploring multi-resolution techniques which are common in signal processing~\cite{mallat2008wavelet}, to deal with higher-dimensional problems, where short-term and long-term interactions could be partially decoupled.
 
\subsection*{Acknowledgements}
We acknowledge support  the European Research
Council (grant SEQUOIA 724063).

{ 
\bibliography{maxplus}
}

\appendix

\section{Value iteration and reduced value-iteration convergence}

In this appendix, we provide short proofs for the lemmas and propositions of the main paper related to (reduced) value iteration.

\subsection{Approximation of the value function through approximate fixed points}
\label{app:lemmaTVV}

This is taken  from~\cite[Prop.~2.1]{bertsekas2012weighted} and presented because the proof structure is used later.

\begin{lemma}[\cite{bertsekas2012weighted}]
\label{lemma:TVV}
If $\| V - T V \|_\infty \leqslant \varepsilon$,  then
  $\| V_\ast - V\|_\infty \leqslant  {\varepsilon}(1-\gamma)^{-1}$.
\end{lemma}
\begin{proof}
Consider 
$\varepsilon_t = \| V - T^t V \|_\infty$. We have
$\varepsilon_t \leqslant \| T T^{t-1} V - T V\|_\infty + \| T V- V\|_\infty \leqslant \gamma \varepsilon_{t-1} + \varepsilon$, because $T$ is $\gamma$-contractive. This leads to
$\varepsilon_t \leqslant \sum_{u=0}^{t-1}\gamma^u \varepsilon \leqslant \frac{\varepsilon}{1-\gamma}$, thus $\| V - V_\ast\|_\infty \leqslant \frac{\varepsilon}{1-\gamma}$ by letting $t$ tend to infinity. 
\end{proof}

\subsection{Proof Prop.~\ref{prop:approx}}

 \label{app:propapprox}
(a) This is consequence of the non-expansiveness of  $\W \W^+ $ and ${{\Z\!}^{\top+}}{{\Z\!}^\top}$, and the $\gamma$-contractiveness of $T$. 
\\
(b) We have: $
\| \hat{T} V_\ast - V_\ast \|_\infty
=\| \W \W^+  {{\Z\!}^{\top+}}{{\Z\!}^\top} V_\ast - V_\ast \|_\infty
\leqslant 
\| \W \W^+  {{\Z\!}^{\top+}}{{\Z\!}^\top} V_\ast - \W \W^+   V_\ast \|_\infty
+ \|  \W \W^+   V_\ast - V_\ast \|_\infty$. Using the non-expansivity of $\W \W^+$, we get
\[
\| \hat{T} V_\ast - V_\ast \|_\infty
 \leqslant 
\|  {{\Z\!}^{\top+}}{{\Z\!}^\top} V_\ast -  V_\ast \|_\infty
+ \|  \W \W^+   V_\ast - V_\ast \|_\infty \leqslant 2 \eta.\]
This imples implies that
$\| V_\infty - V_\ast \|_\infty \leqslant \frac{2 \eta }{1-\gamma}$ using the same reasoning as in the proof of Lemma~\ref{lemma:TVV}.
 
 The result extends to assumptions on $\min_{\alpha \in \rb^\W} \| \W \alpha - V_\ast \|_\infty$ instead of $\|  \W \W^+   V_\ast - V_\ast \|_\infty$ (and similarly for $\Z$). Indeed, we have, with $\tilde{\alpha} =   \W^+   V_\ast $,
 \BEAS
 \min_{\alpha \in \rb^\W} \| \W \alpha - V_\ast \|_\infty
&  \leqslant &   \| \W \tilde{\alpha} - V_\ast \|_\infty = \|  \W \W^+   V_\ast - V_\ast \|_\infty, \mbox{ and  for any } \alpha, \\
\|  \W \W^+   V_\ast - V_\ast \|_\infty
 & \leqslant & \|  \W \W^+   V_\ast  - \W \alpha \|_\infty + \|\W \alpha  - V_\ast \|_\infty
 \leqslant \|   \W^+   V_\ast  -  \alpha \|_\infty + \|\W \alpha  - V_\ast \|_\infty
 \EEAS
 by non-expansivity of $\W$. By taking the infimum over $\alpha$, we get that
 \[
 \min_{\alpha \in \rb^\W} \| \W \alpha - V_\ast \|_\infty \leqslant  |  \W \W^+   V_\ast - V_\ast \|_\infty \leqslant 2 \min_{\alpha \in \rb^\W} \| \W \alpha - V_\ast \|_\infty .
 \]
 
 In order to prove the result when replacing $T$ by $T^\rho$, we simply have to notice that the contraction factor of $\tau^\rho$ is $\gamma^\rho$ and that for $\tau = ( 1- \gamma)^{-1} \geqslant 1$ and $\rho \geqslant 1$, we have
 \[
 \frac{1}{1 - ( 1- 1/\tau)^\rho} \leqslant ( 1 + \tau/\rho).
 \]

 \section{Approximation properties of max-plus-linear combinations}
Here we provide proofs for approximation properties of max-plus linear combinations.

\subsection{Proof of Prop.~\ref{prop:constant} (piecewise-constant functions)}
\label{app:propconstant}

We have, for any $s \in \S$:
\BEAS
V(s) - \W\W^+ V(s)  &  = &     
V(s) - \max_{w \in \W} w(s) + \min_{s' \in \S} V(s')  -w(s'),
\EEAS
which is thus always non-negative (and valid for any family $\W$ of functions).

Thus, for partition-based set of functions, if $w$ is chosen so that $s \in A(w)$, then,  $w$ is the maximizer above and $s'$ above is restricted to $A(w)$, because the value of $w$ is $-\infty$ outside of $A(w)$. Thus

\BEAS
V(s) - \W\W^+ V(s)  &  = &     
V(s)  - \min_{s' \in A(w) } V(s')  \\
&= & \max_{s' \in A(w)} V(s') - V(s)\leqslant \max_{s' \in A(w)} | V(s') - V(s) |
\\
& \leqslant &   {\rm Lip}_p(V) \max_{s',s'' \in A(w) } d(s',s'')^p.
\EEAS
The result follows from the fact that with the choice of partition, $\max_{s',s'' \in A(w) } d(s',s'')$ is less than $ 2\eta(n,\S)$.

\subsection{Proof of Prop.~\ref{prop:dist} (distance functions)}
\label{app:propdist}

The proof is similar to~\cite{akian2008max} (but slightly tighter). First, we have, for
 $c \geqslant {\rm Lip}_1(V)$, and any $w \in \S$:
\[V(w) =   \min_{s' \in \S} V(s') + c \cdot d(s',w) .\]
 Indeed, (a) $V(w)$ is equal to $V(s') + c \cdot d(s',w)$ for $s'=w$, which implies $V(w) \geqslant  \min_{s' \in \S} V(s') + c \cdot d(s',w)$;
 moreover, (b) $V(s') + c \cdot d(s',w) \geqslant V(w)-  {\rm Lip}_1(V) \cdot d(s'w) + c \cdot d(s',w) \geqslant V(w)$, for all $s' \in \S$, which implies that  $V(w) \leqslant  \min_{s' \in \S} V(s') + c \cdot d(s',w)$.
 
 Therefore, we get:
\BEAS
\| V - \W\W^+ V\|_\infty & = &  \textstyle \max_{s \in \S}
V(s) - \max_{w \in \W} -c \cdot d(w,s) + \min_{s' \in \S} V(s') + c \cdot d(s',w) \\
& = & \textstyle \max_{s \in \S} \min_{w \in \W}
V(s) + c \cdot d(w,s) - V(w)  \\
&  \leqslant  & 
  \max_{s \in \S} \min_{w \in \W}  c \cdot d(w,s)  + {\rm Lip}_1(V) \cdot d(w,s)  \\
  & \leqslant & 2c \cdot  \max_{s \in \S} \min_{w \in \W}    d(w,s).
\EEAS

\subsection{Bregman basis functions smooth functions}
\label{sec:quadratic}
\label{app:bregman}

In this section, we consider approximations of functions on a subspace of $\rb^d$. These results can be used for discretizations of smooth problems.

Assuming that $\S$ is a convex compact subset of $\rb^d$, and $h$ is any convex function from $\rb^d$ to $\rb$, then, we consider the functions of the form $-\mathcal{D}_h(s,v) = - h(s) + h(v) - h'(v)^\top(s-v)$, which is the negative Bregman divergence associated to $h$, which is an extension of \cite{akian2008max} from quadratic to all convex functions $h$. Since the functions need only be defined up to constants, we can reparameterize them with $w=-h'(v)$, and thus consider
$w(s) = - h(s) + w^\top s$,  where $\W$ is identified to a convex subset of $\rb^d$.  Then $\W\W^+ V$ is related to $(V+h)^{\ast \ast}-h$, where $g^*$ is the Fenchel-conjugate~\cite{boyd2004convex} of a function $g$. More precisely, we have (see proof in App.~\ref{app:proofbregman}).
\begin{proposition}
\label{prop:bregman}
(a) If $\W$ contains the domain of $(V+h)^\ast$, then $\W\W^+ V = (V+h)^{\ast \ast} - h$.
\\
(b) More generally, for any norm $\Omega$ on $\rb^d$,
\[  \textstyle \| (V+h)^{\ast \ast} - h  - \W\W^+ V\|_\infty \leqslant 
{\rm diam}_\Omega(\S)
\max_{w \in {\rm domain}(( V+h)^\ast) } \min_{w' \in \W}  \Omega^\ast(w - w').\]
(c) If $(V+h)^{\ast \ast} - g$ is convex, for a convex function $g$, then:
\[  \textstyle  \| (V+h)^{\ast \ast} - h  - \W\W^+ V\|_\infty \leqslant 
\max_{w \in {\rm domain}(( V+h)^\ast) }  \min_{w' \in \W}  \mathcal{D}_{g^\ast}(w',w).\]
\end{proposition}

Thus, when $\W$ is not discretized, projection onto the image space of $\W$ corresponds to projection on the set of functions such that $V+h$ is convex: indeed, (a) implies that if $V+h$ is convex, $(V+h)^{\ast \ast} = V+h$ and  $\W\W^+ V =V$. When $\W$ is discretized, then we get an approximation of the result above, with an approximation error that vanishes when $\W$ covers the domain of $( V+h)^\ast$.
Similarly, for $\Z$, we obtain a similar behavior but for concave functions, because $  {{\Z\!}^{\top+}}{{\Z\!}^\top}  V = - \Z \Z^+ ( - V)$.

\paragraph{Smooth functions.}  If we make the assumption that the function $V$ is smooth with respect to the Bregman divergence defined from the convex function $\frac{1}{2} h$~\cite{bauschke2016descent}, that is, for all $s,s'$,
$-\frac{1}{2}\mathcal{D}_h(s',s) \leqslant V(s') - V(s) - \nabla V(s)^\top (s'-s) \leqslant \frac{1}{2}\mathcal{D}_h(s',s)$, that is, $V + \frac{1}{2} h$ and $\frac{1}{2}h - V$ are convex; this implies that with $g = \frac{1}{2} h$,
$(V+h)^{\ast \ast}  - g = V+h - \frac{1}{2} h = V+\frac{1}{2} h$ is convex, and thus statement (c) from Prop.~\ref{prop:bregman} leads to approximation guarantee for $\| V   - \W\W^+ V\|_\infty$. Similarly, the fact that $\frac{1}{2}h - V$ is convex leads to a similar guarantee for  $\| {{\Z\!}^{\top+}}{{\Z\!}^\top}  V - V\|_\infty$.

Morever, with $g = \frac{1}{2} h$, if $h$ is strongly-convex, then $g^\ast$ is smooth and we get a squared norm $\Omega(w-w')^2$ in the guarantees in Prop.~\ref{prop:bregman}.

In order to obtain guarantees from a finite number of basis functions, we just need a cover of the state-space for the Bregman divergence.

\subsection{Proof of Prop.~\ref{prop:bregman}}
\label{app:proofbregman}
The proof follows the same structure as \cite{akian2008max}, but extended to all convex functions $h$ (and not only quadratic).
We have, by definition of the Fenchel-conjugate $(V+h)^\ast$ of $V+h$ (see~\cite{boyd2004convex}):
\BEA
 \nonumber \W\W^+ V(s) & = &  \max_{w \in \W} w^\top s - h(s) + \min_{s' \in \S} V(s') + h(s') - w^\top s' \\
\label{eq:breg} & = & \max_{w \in \W} w^\top s - h(s) - (V+h)^\ast(w).
\EEA

Thus,  if $\W$ contains the domain of $(V+h)^\ast$, we have $\W\W^+ V(s) = (V+h)^{\ast \ast}(s) - h(s)$, which shows (a). Moreover, this implies that in all situations, we have  $ \W\W^+ V(s) \leqslant  (V+h)^{\ast \ast}(s) - h(s)$.

We now denote $w^\ast(s)$ the unconstrained minimizer in the optimization problem in \eq{breg} but do not assume that $\W$ contains the domain of $(V+h)^\ast$. Moreover, since $\S$ is compact, the function $(V+h)^\ast$ has subgradients in $\S$, thus the function
$w \mapsto w^\top s  -  (V+h)^\ast(w)$ has gradients bounded in the norm $\Omega$ by ${\rm diam}(\S)$. We thus get:
\BEAS
\!\!\!  (V+h)^{\ast \ast}(s) - h(s) - \W\W^+ V(s)
 &\!\!\!= \!\!\!& \min_{w \in \W} - w^\top s  + (V+h)^\ast(w) - w^\ast(s)^\top s -(V+h)^\ast(w^\ast(s)) \\
 & \!\!\!\leqslant\!\!\! & {\rm diam}(\S) \min_{w \in \W}  \Omega^\ast(w - w^\ast(s) ).
\EEAS
This leads to (b).

Since $w^\ast(s)$ is the unconstrained maximizer of $w^\top s - h(s) - (V+h)^\ast(w)$, it is such that
$s \in \partial  (V+h)^\ast(w^\ast(s))$. Since $(V+h)^{\ast  \ast} - g$ is convex, $g^\ast - (V+h)^{\ast }$ is convex, and then, for any $w$, 
\BEAS
g^\ast(w) - (V+h)^{\ast  }(w) & \geqslant  & 
g^\ast(w^\ast(s)) - (V+h)^{  \ast}(w^\ast(s))
+ ( w - w^\ast(s))^\top   \big[ \nabla g^\ast(w^\ast(s) ) - s \big],
\EEAS
leading to, by rearranging terms:
\[
 - (V+h)^{\ast  }(w)  \geqslant -\mathcal{D}_{g^\ast}(w,w^\ast(s)) - (V+h)^{  \ast}(w^\ast(s))
-  ( w - w^\ast(s))^\top s.
\]
This leads to, for any $w \in \W$,
\[
  (V+h)^{\ast  }(w) - w^\top s -
  (V+h)^{  \ast}(w^\ast(s)) +  w^\ast(s)^\top s \leqslant \mathcal{D}_{g^\ast}(w,w^\ast(s)).
\]
Taking the infimum with respect to $w \in \W$, we get:
\[
- \W\W^+ V(s) - h(s)
 + (V+h)^{\ast \ast}(s)  
 \leqslant   \min_{w \in \W}   \mathcal{D}_{g^\ast} (w,w^\ast(s)),
 \]
 which in turn leads to (c), since the two quantities above are non-negative.

\section{Detailed proofs of complexities of iterations}
\label{app:comp}
\paragraph{Compile time for sparse 1D graph.}
Given an order $k$ chain graph, the square of the adjacency matrix is of order $2k$ and getting it can be done in $O( k^2 |\S|)$ (with $2k|\S|$ elements to obtain, each with $O(k)$ complexity). Thus, the overall complexity is
$\sum_{i=0}^{\log_2 \! \rho} (2^{i})^2 |\S| = O( |\S| \rho^2) $.

\paragraph{Compile time for sparse  graph (general dimension).}
Given an order $k$ $d$-dimensional grid (where each node is connected to $2d$ neighbors, two per dimension), up to multiplicative constants depending on $d$, the square of the adjacency matrix is of order\footnote{This degree ${\rm deg}(\rho,d)$ is equal to the number of elements of the elements of the $\ell_1$-ball of radius $\rho$ with integer coordinates. Following~\cite{serra2000enumeration}, we have
$ {\rm deg}(\rho,d) = \sum_{i=0}^{ \min \{ d, \rho \}} 2^i { d \choose i} {\rho \choose i}$.
In particular ${\rm deg}(\rho,d)  = {\rm deg}(d,\rho) $, and we have the bound
$ {\rm deg}(\rho,d) \leqslant \rho^d \frac{2^d}{d!}$, which is the volume of the $\ell_1$-ball of radius $\rho$ in dimension $d$, which grows as $\rho^d$.
}
 $ k^d$ and getting it can be done in $O( k^{2d} |\S|)$ (with $k^d|\S|$ elements to obtain, each with $O(k^d)$ complexity). Thus, the overall complexity is
$\sum_{i=0}^{\log_2 \! \rho} (2^{i})^{2d} |\S| = O( |\S| \rho^{2d}) $.

\section{Convex optimization for estimating $z_{\rm new}$ and $w_{\rm new}$}

\label{app:cvxrelax}

In order to go beyond a finite set of functions, one can optimze $z_{\rm new}$ as follows (the optimization problem for $w_{\rm new}$ would follow similarly).

We can write the criterion that needs to be optimized from \mysec{greedy} as follows:
\BEAS
\!\!\!\!\!& & \!\!\!\!\max_{s \in \S} \min \big\{
U(s) - TV(s), - TV(s) - z_{\rm new}(s)   + \langle TV |z_{\rm new} \rangle,
\big\}
\\
\!\!\!\!\!& = &\!\!\!\! \max_{s \in \S} \min_{ \eta(s,z_{\rm new}) \in [0,1]}  \! \eta(s,z_{\rm new}) \big[
U(s) \!-\! TV(s)\big] + ( 1 \!-\! \eta(s,z_{\rm new})) \big[ \!-\! TV(s)\! -\! z_{\rm new}(s)   \!+\! \langle TV |z_{\rm new} \rangle \big]
\\
\!\!\!\!\!& \leqslant &\!\!\!\! \max_{s \in \S}    \eta^\ast(s,\tilde{z}_{\rm new}) \big[
U(s) - TV(s)\big] + ( 1-  \eta^\ast(s,\tilde{z}_{\rm new})) \big[ - TV(s) - z_{\rm new}(s)   + \langle TV |z_{\rm new} \rangle \big],
\EEAS
where $ \eta^\ast(s,\tilde{z}_{\rm new})$ is the minimizer for a fixed $\tilde{z}_{\rm new}$. The previous function is \emph{convex} in $z_{\rm new}$.
This leads to a natural majorization-minimization algorithm, which needs to be initialized in a problem dependent way, and can thus be used to  learn linear parametrizations of $z_{\rm new}$.

\section{Extra experiments}
In this section, we provide extra experiments and complements on one-dimensional and two-dimensional problems.
\label{app:exp}

\subsection{One-dimensional}

We consider pairs $(b,V)$ of reward and optimal value functions that satisfy the Hamilton-Jacobi-Bellman (HJB) equation~\cite{crandall1983viscosity,munos2000study}:
$V(x) \log \eta +  |V'(x)|  + b(x)= 0$. We consider two functions~$V(x)$, from which we can recover the function $b(x)$:
\BIT
\item $V(x) = (1  - 3 x)_+ + (6x - 4  )_+   + 
 ( 1  - 36 ( x - 1/2  )^2)_+ $, as plotted in \myfig{fixed_basis}.
 
 \item $V(x) = (1  - 3 x)_+ + (6x - 4  )_+$, as plotted in \myfig{cvx_1d}.

\EIT
We consider $\eta = 1/2$, a number of nodes equal to $S= 362 \approx 2^{17/2}$, such that $\gamma \approx 0.9981$ and $\tau \approx 521$.

\begin{figure}
\begin{center}
\hspace*{-.5cm}
\includegraphics[width=3.75cm]{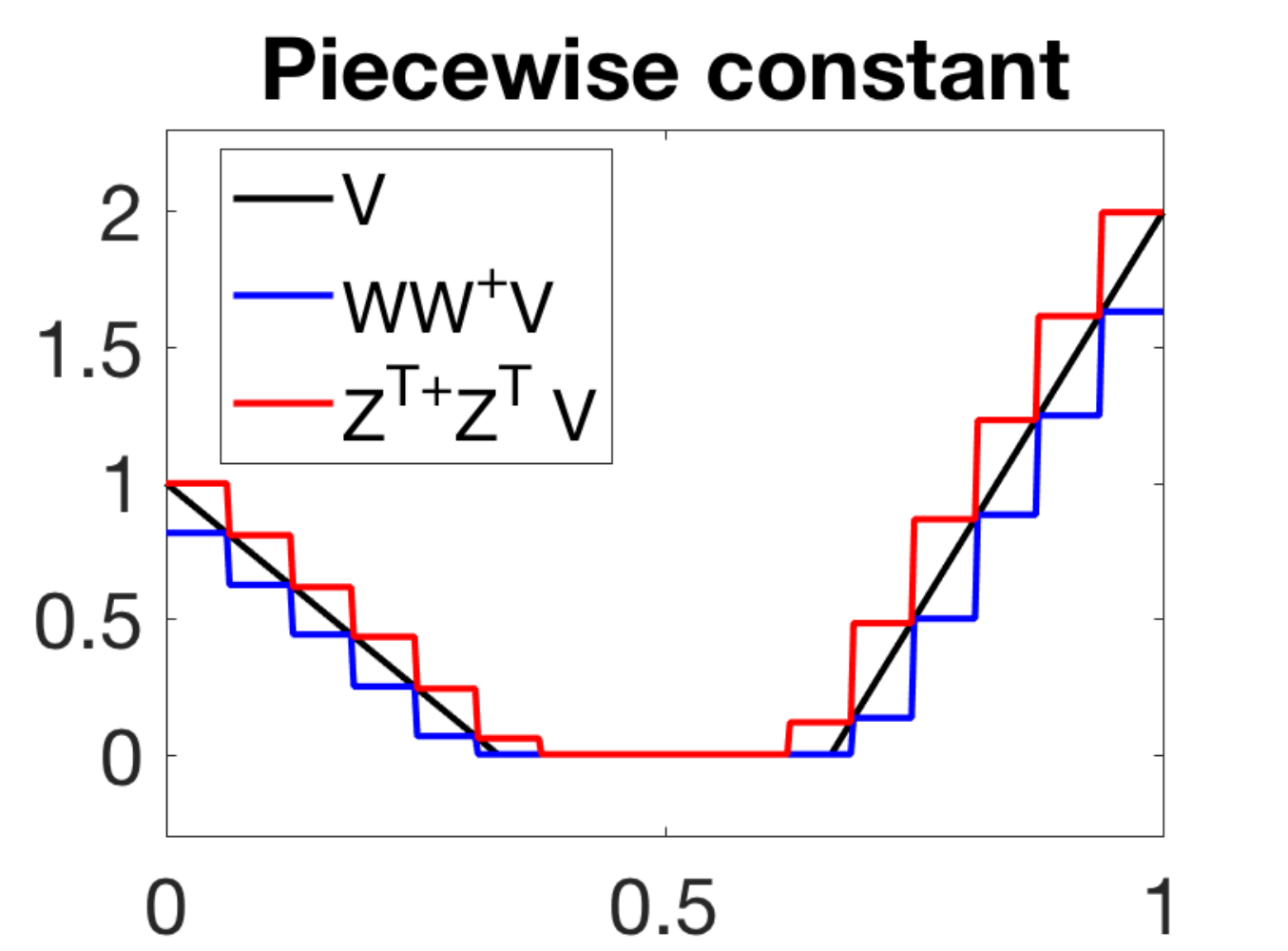} \hspace*{.234cm}
\includegraphics[width=3.75cm]{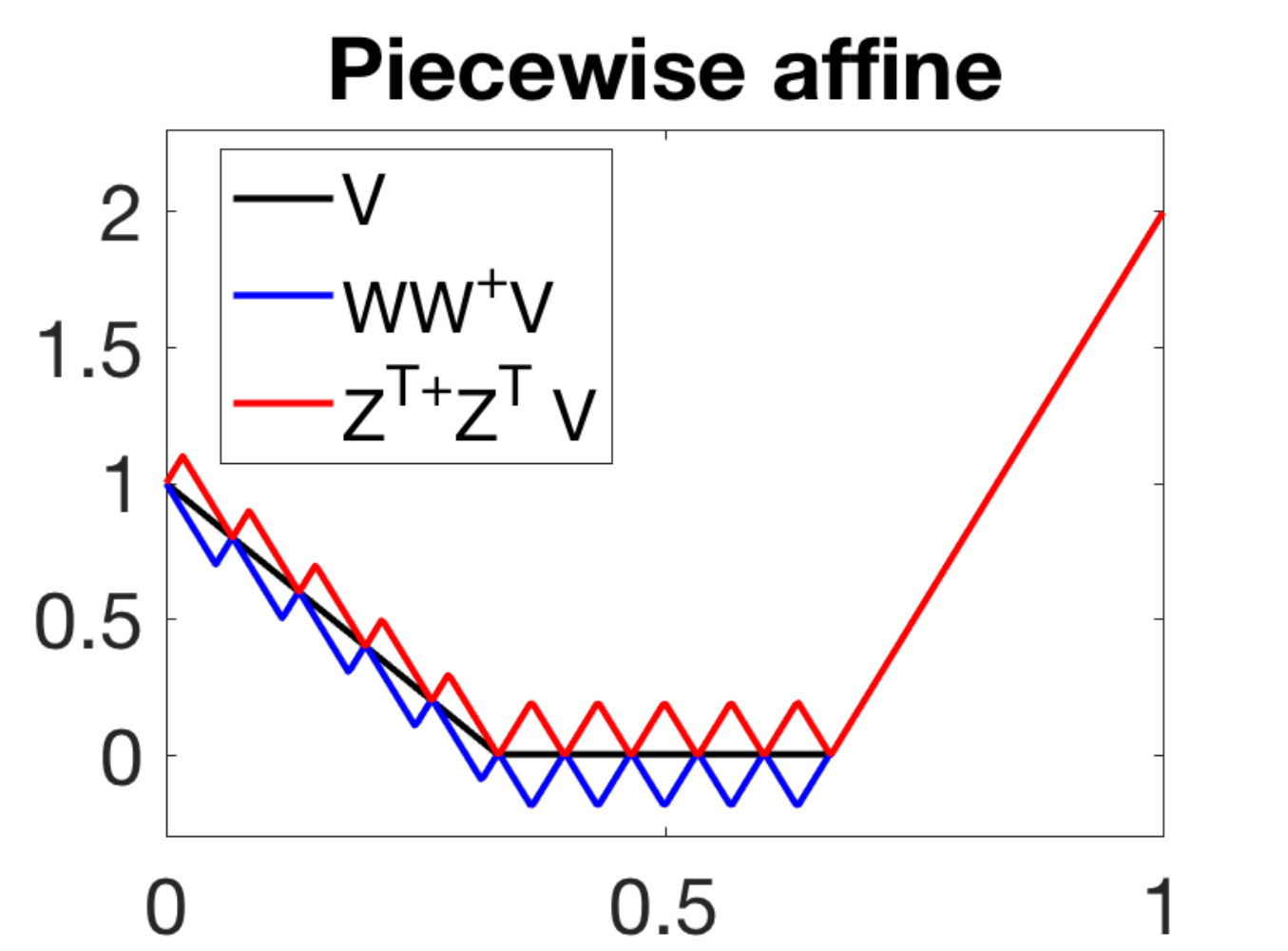} \hspace*{.234cm}
\includegraphics[width=3.75cm]{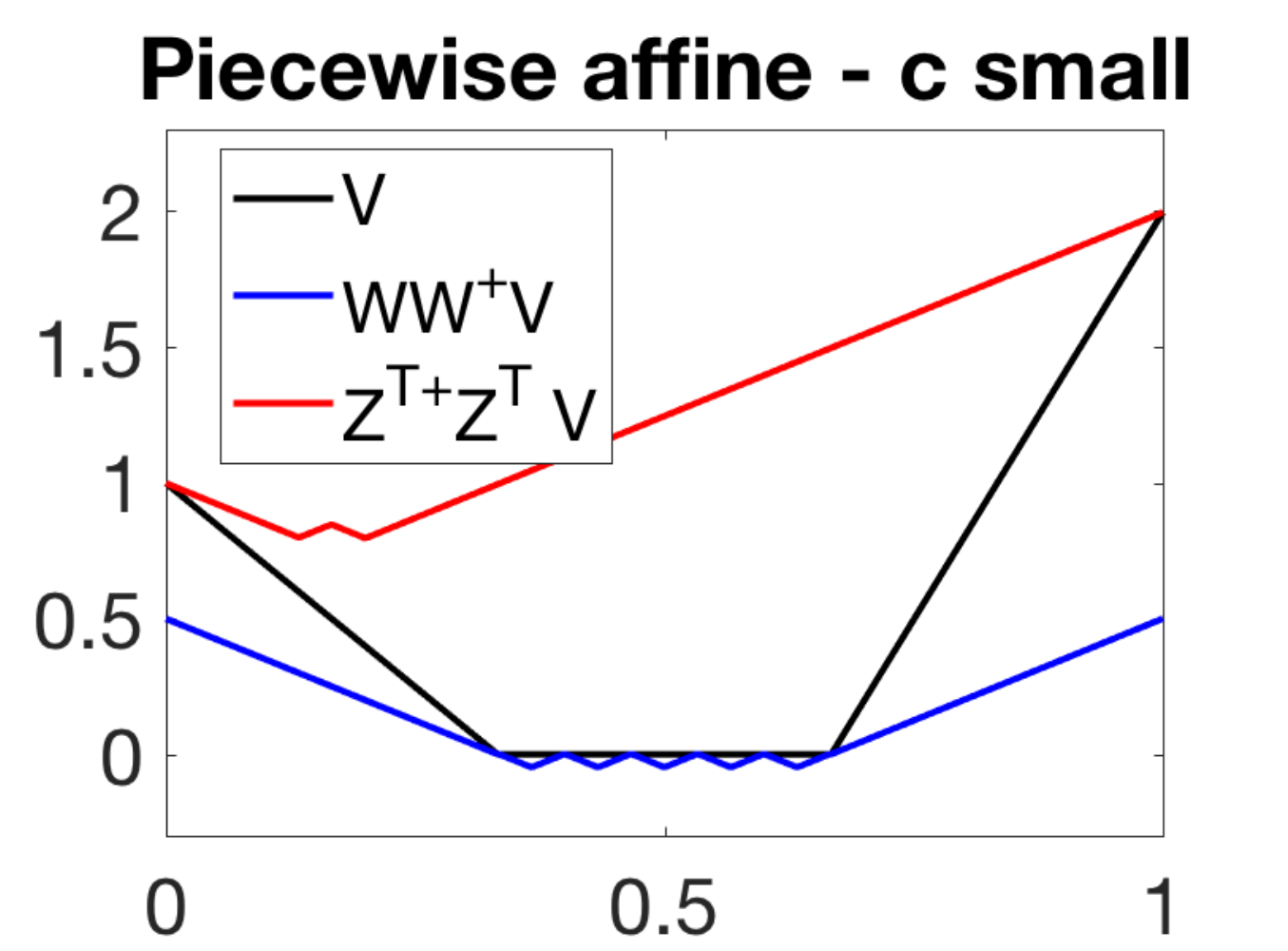} \hspace*{.234cm}
\includegraphics[width=3.75cm]{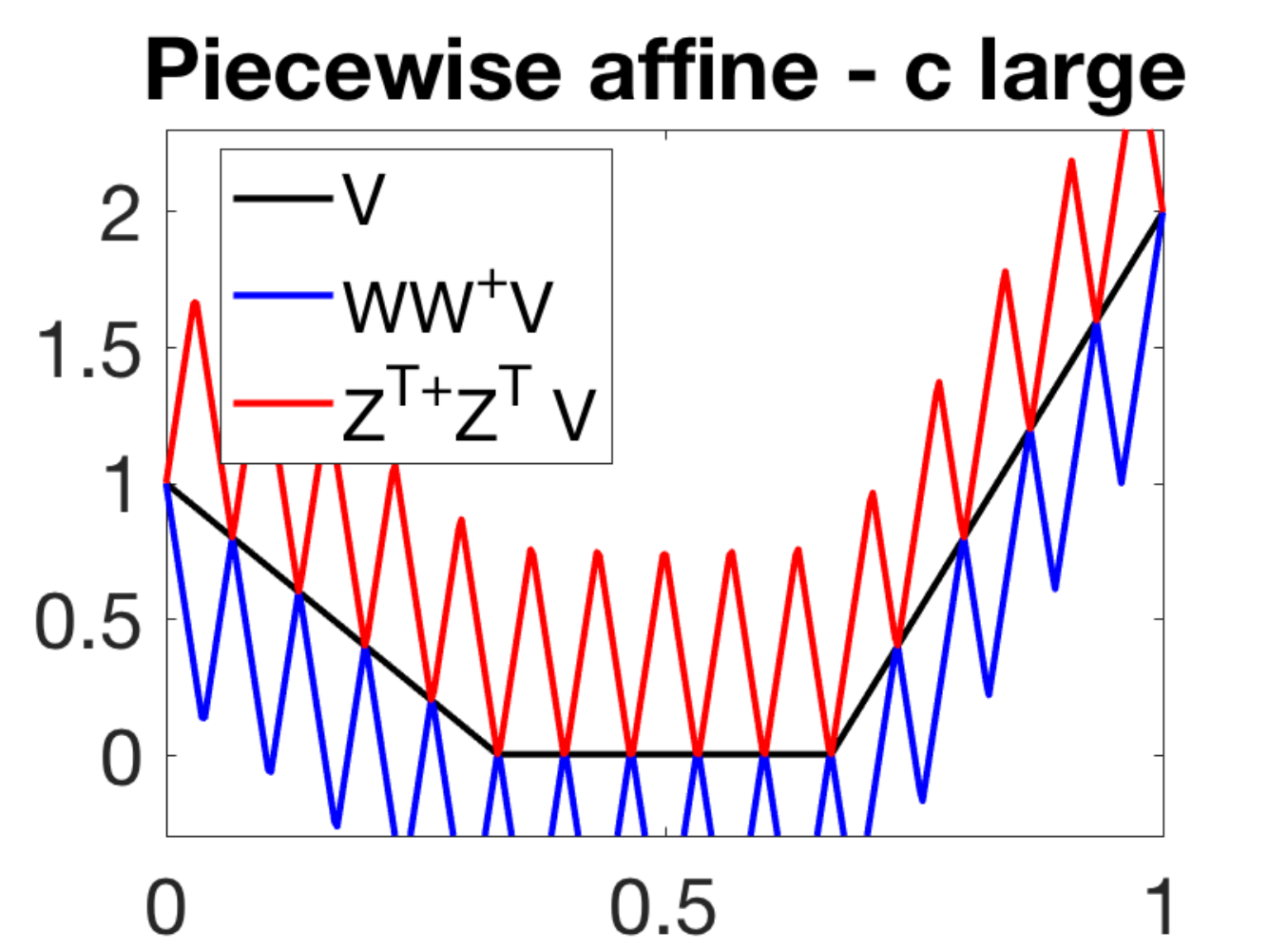}
\hspace*{-.5cm}\end{center}

\vspace*{-.45cm}

\caption{Approximation of a function $V$ with different finite basis with 16 elements. One-dimensional case (with a convex optimal value function).  From left to right: piece-wise constant basis function, piecewise affine basis functions with well chosen value of~$c$, then too small and too large. \label{fig:cvx_1d}}
\end{figure}

\begin{figure}
\begin{center}
\hspace*{-.5cm}
\includegraphics[width=3.75cm]{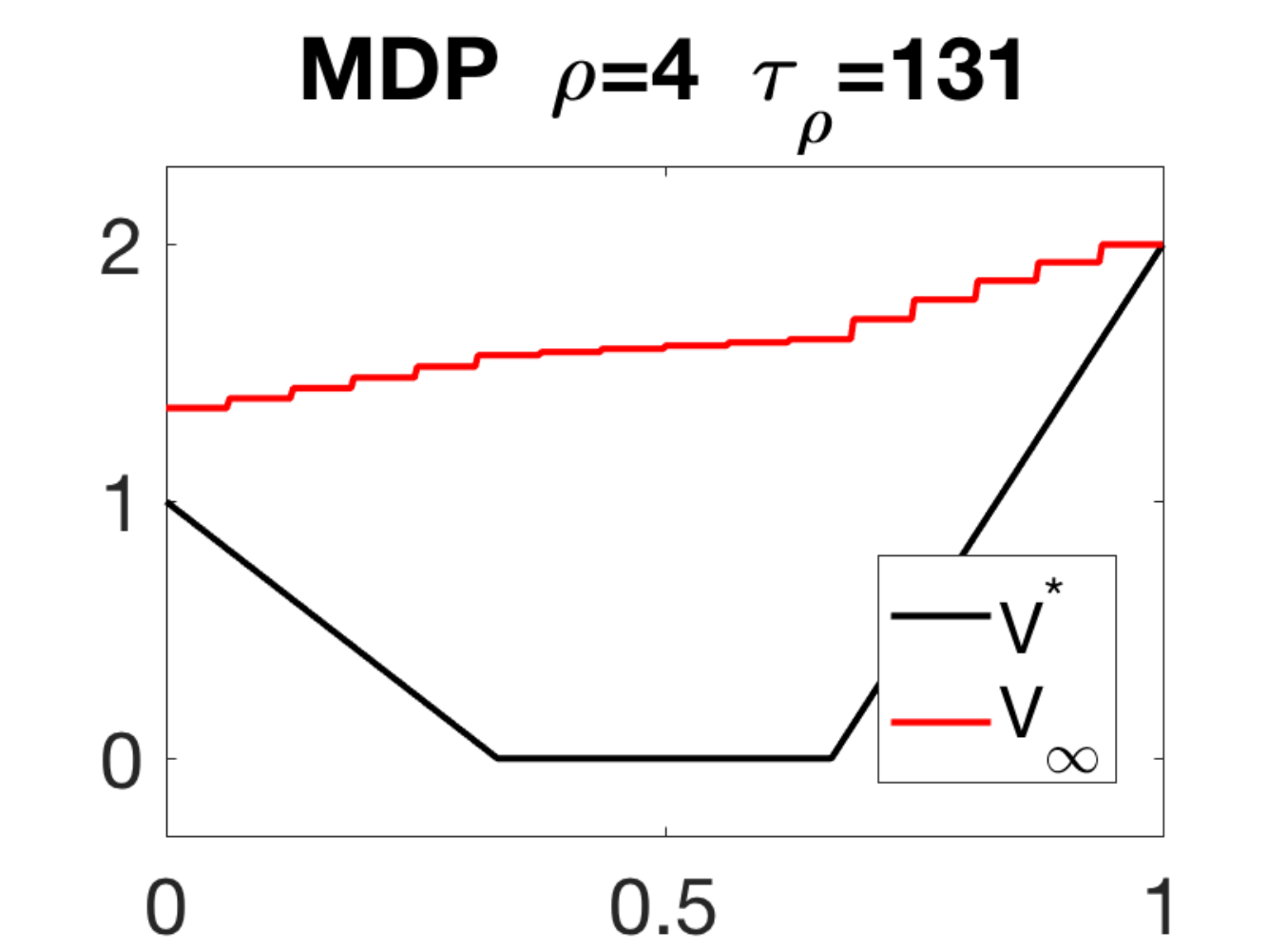} \hspace*{.234cm}
\includegraphics[width=3.75cm]{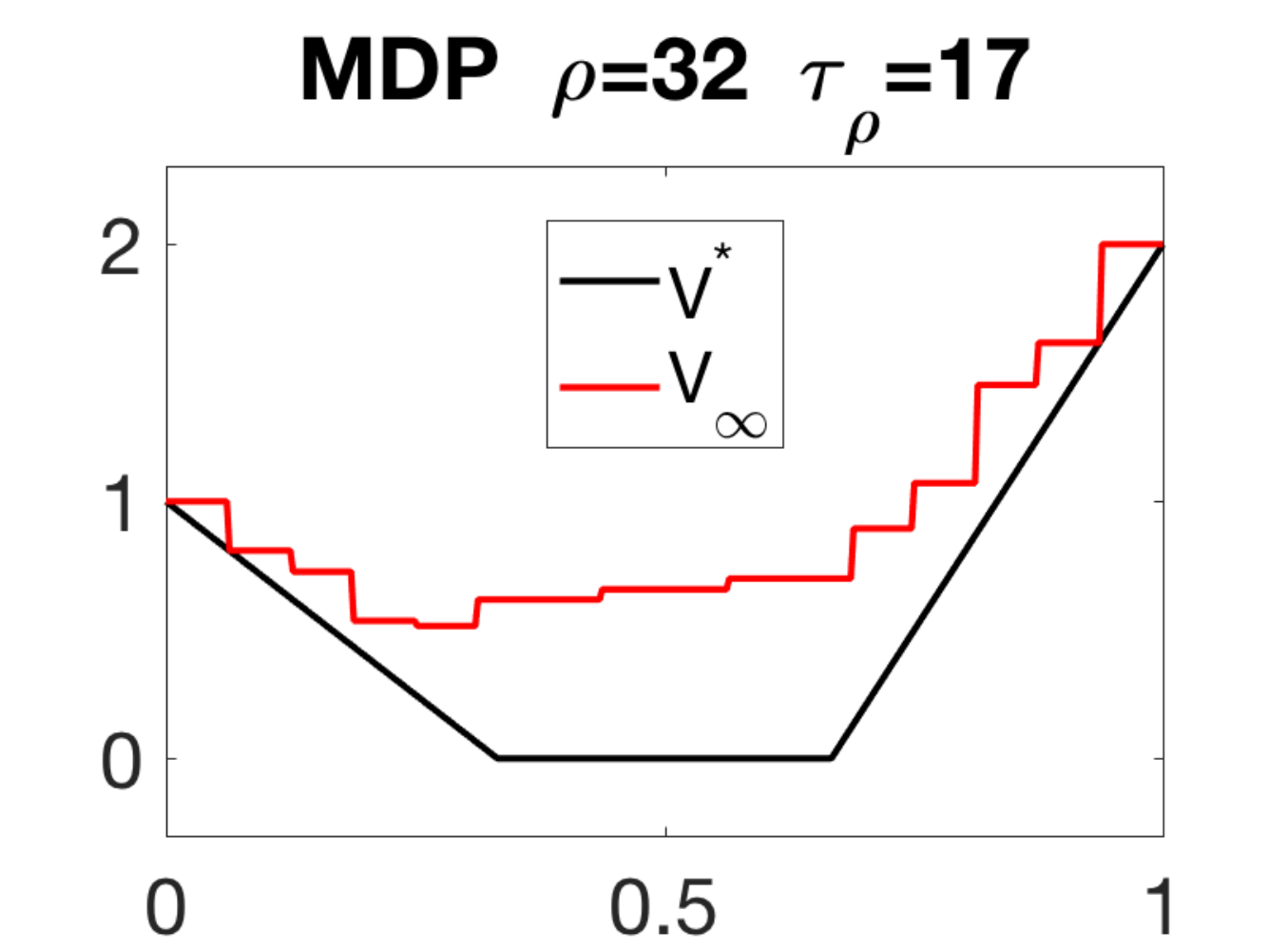} \hspace*{.234cm}
\includegraphics[width=3.75cm]{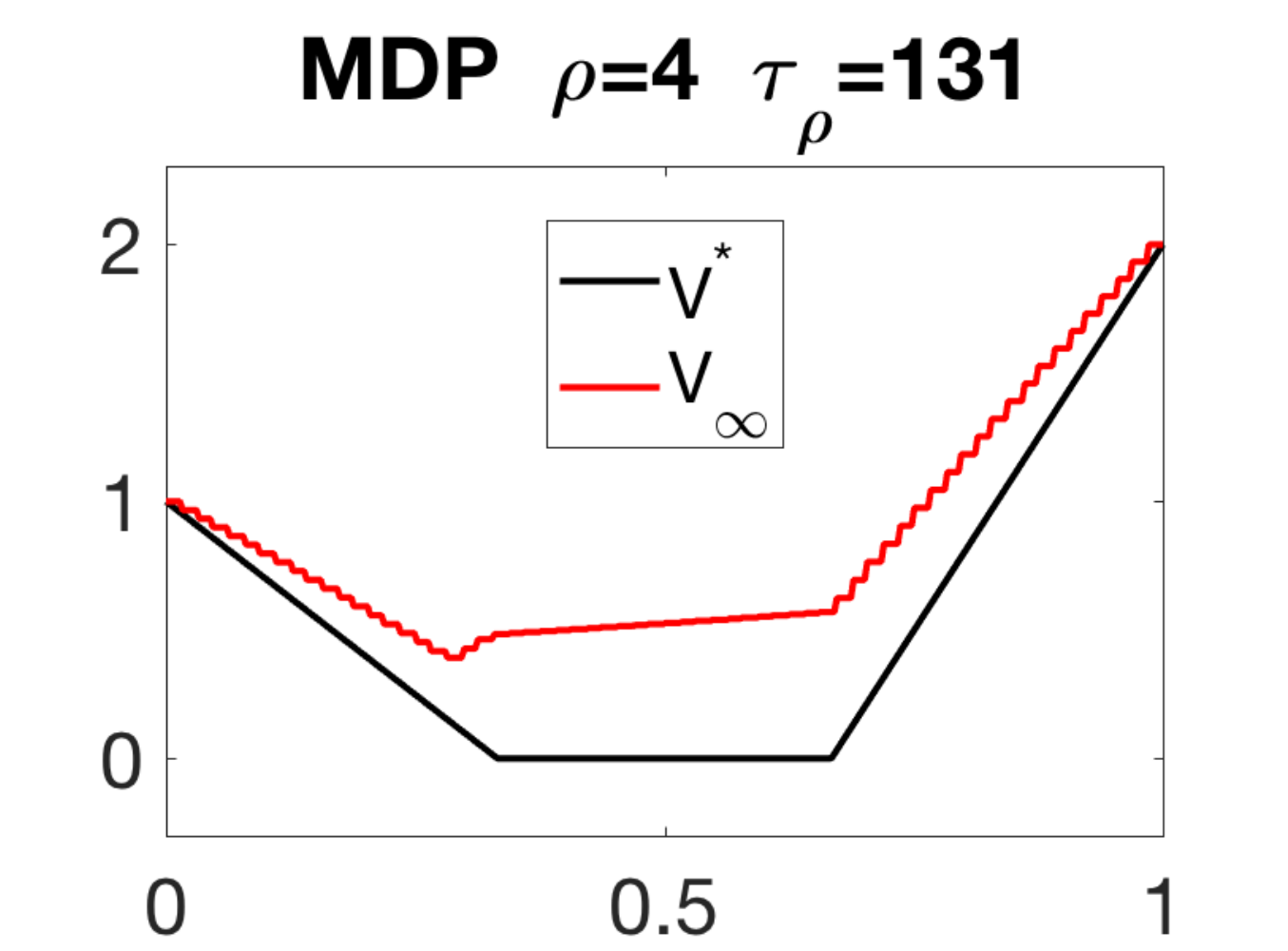} \hspace*{.234cm}
\includegraphics[width=3.75cm]{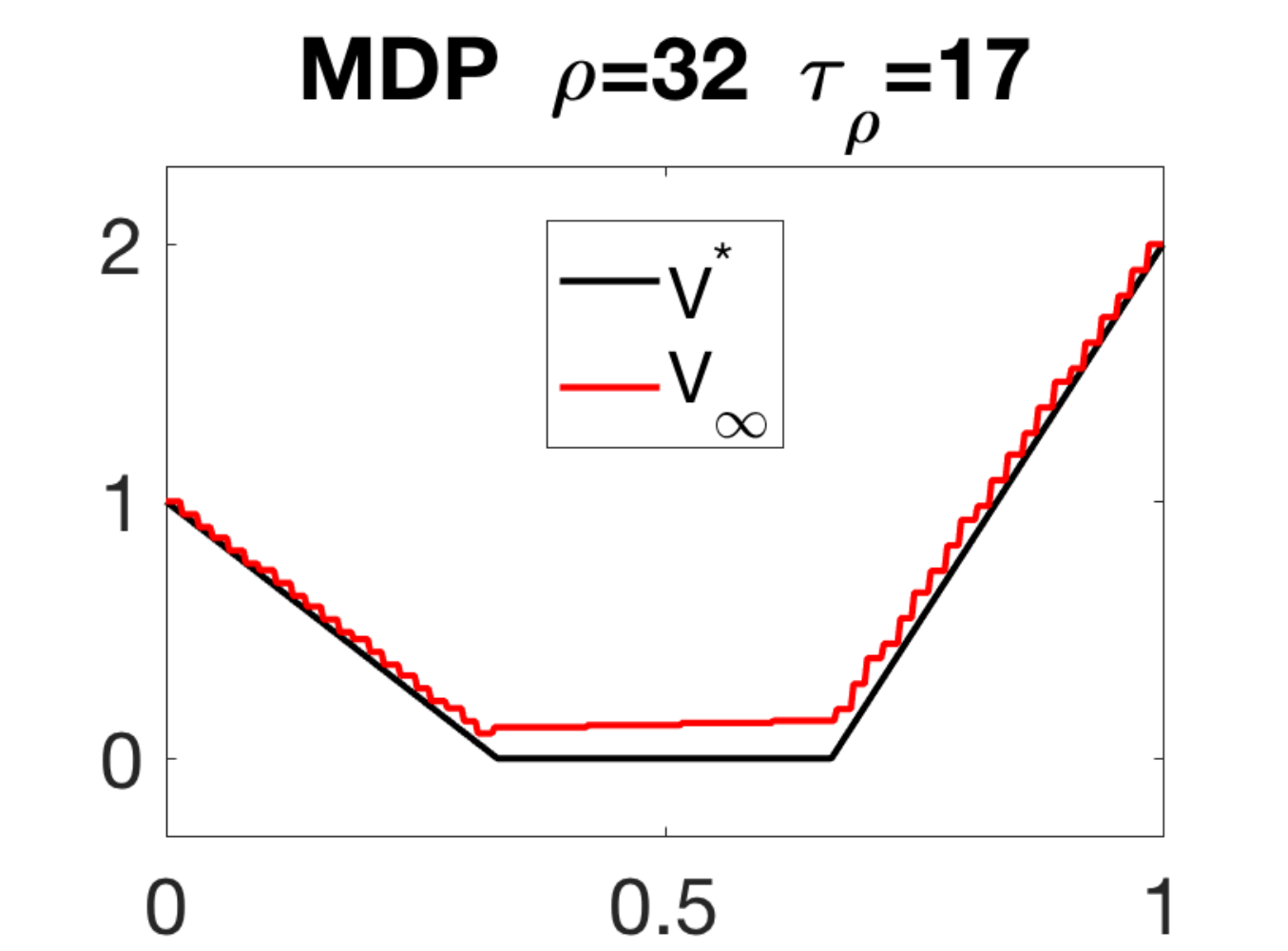}
\hspace*{-.5cm}\end{center}

\vspace*{-.45cm}

\caption{Approximation of a function $V$ with different finite basis with 16 or 64 elements, \emph{within the MDP}, and with values of $\rho$ that are $4$ and $32$. One-dimensional case (with a convex optimal value function).  From left to right: $(n=16, \rho=4)$, $(n=16, \rho=32)$, $(n=64, \rho=4)$, $(n=16, \rho=32)$. \label{fig:cvx_mdp_1s}}

\end{figure}

\begin{figure}
\begin{center}

\hspace*{-.5cm}
\includegraphics[width=3.86cm]{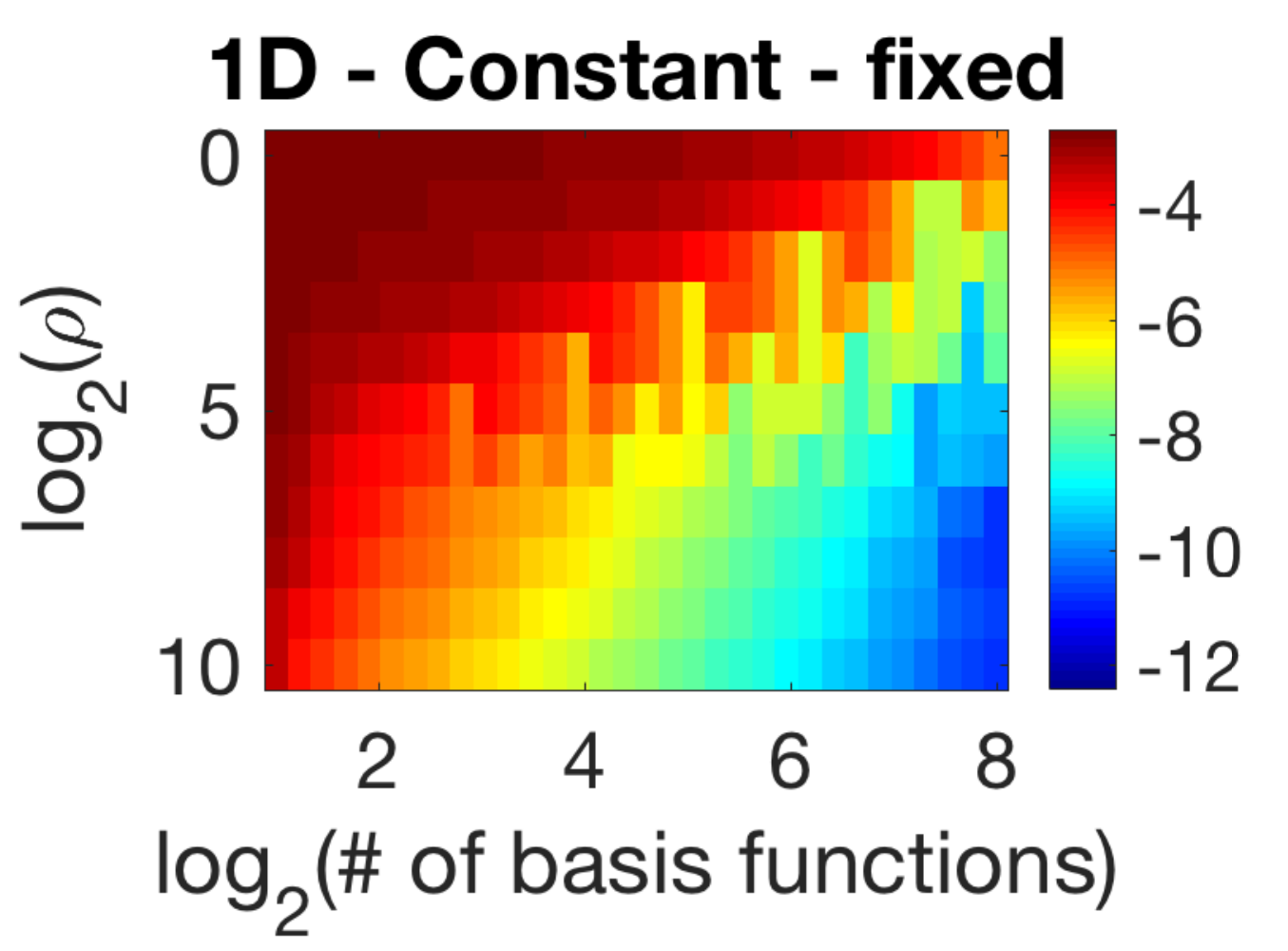} \hspace*{.1234cm}
\includegraphics[width=3.86cm]{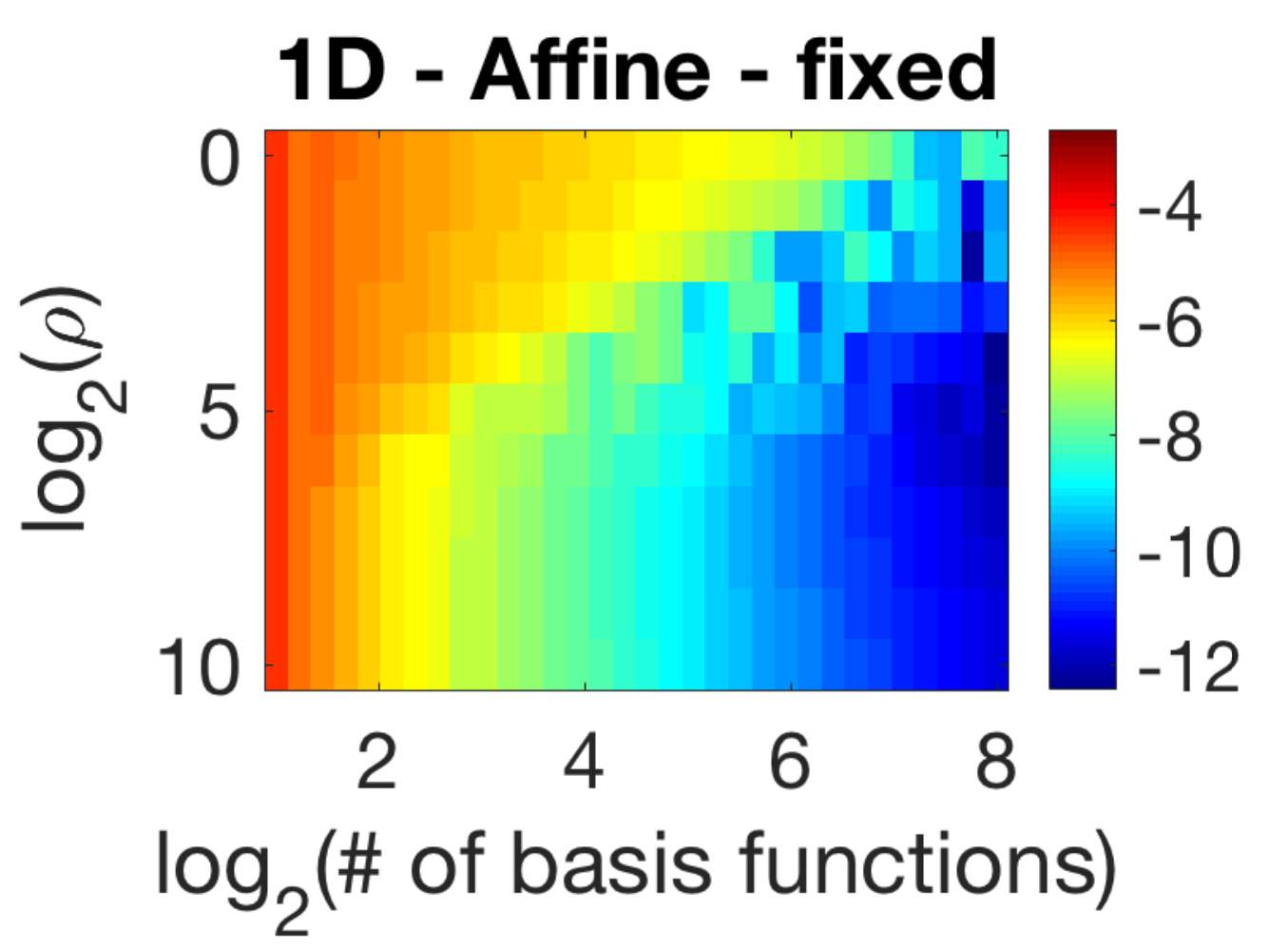} \hspace*{.1234cm}
\includegraphics[width=3.86cm]{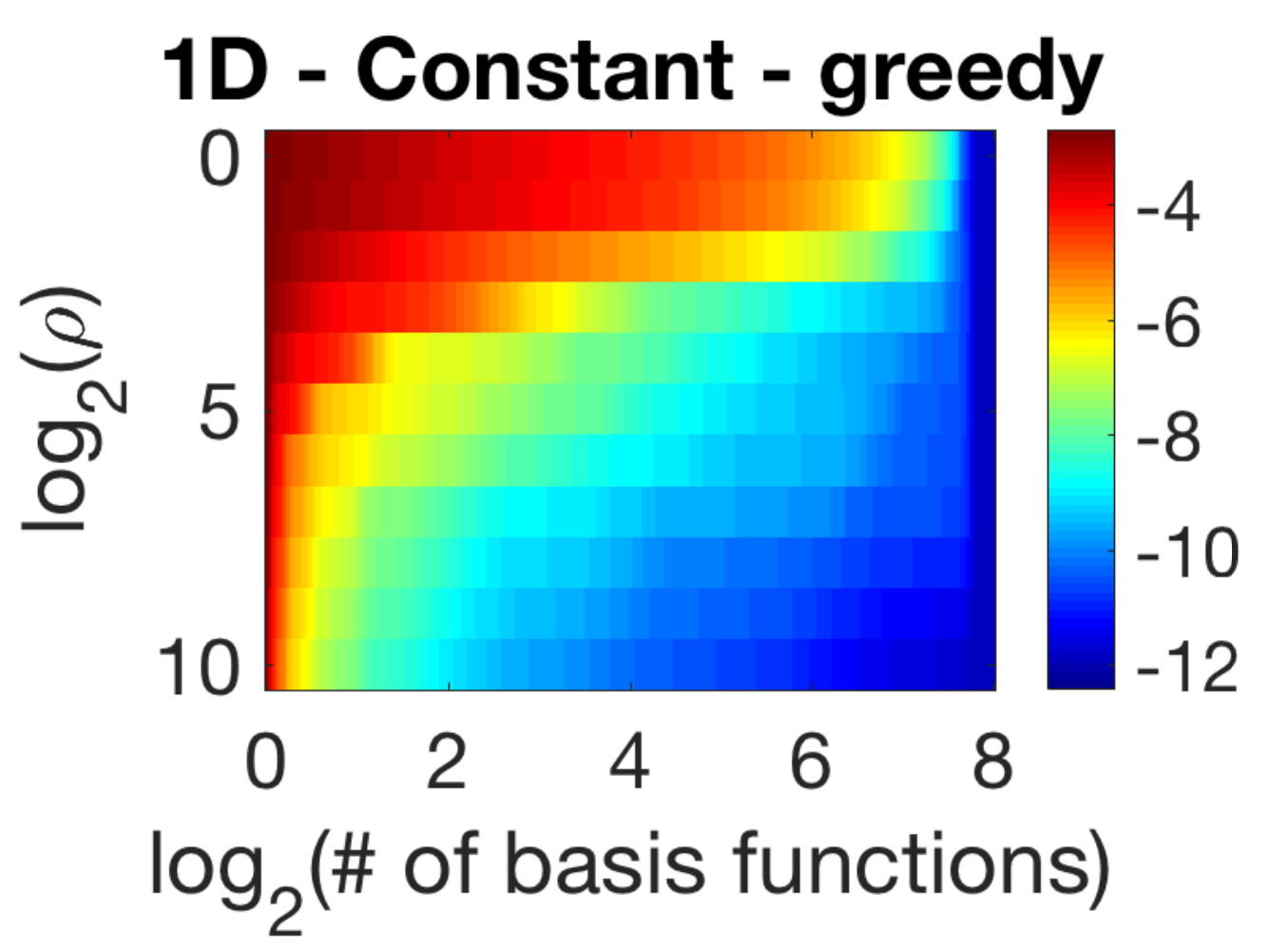} \hspace*{.1234cm}
\includegraphics[width=3.86cm]{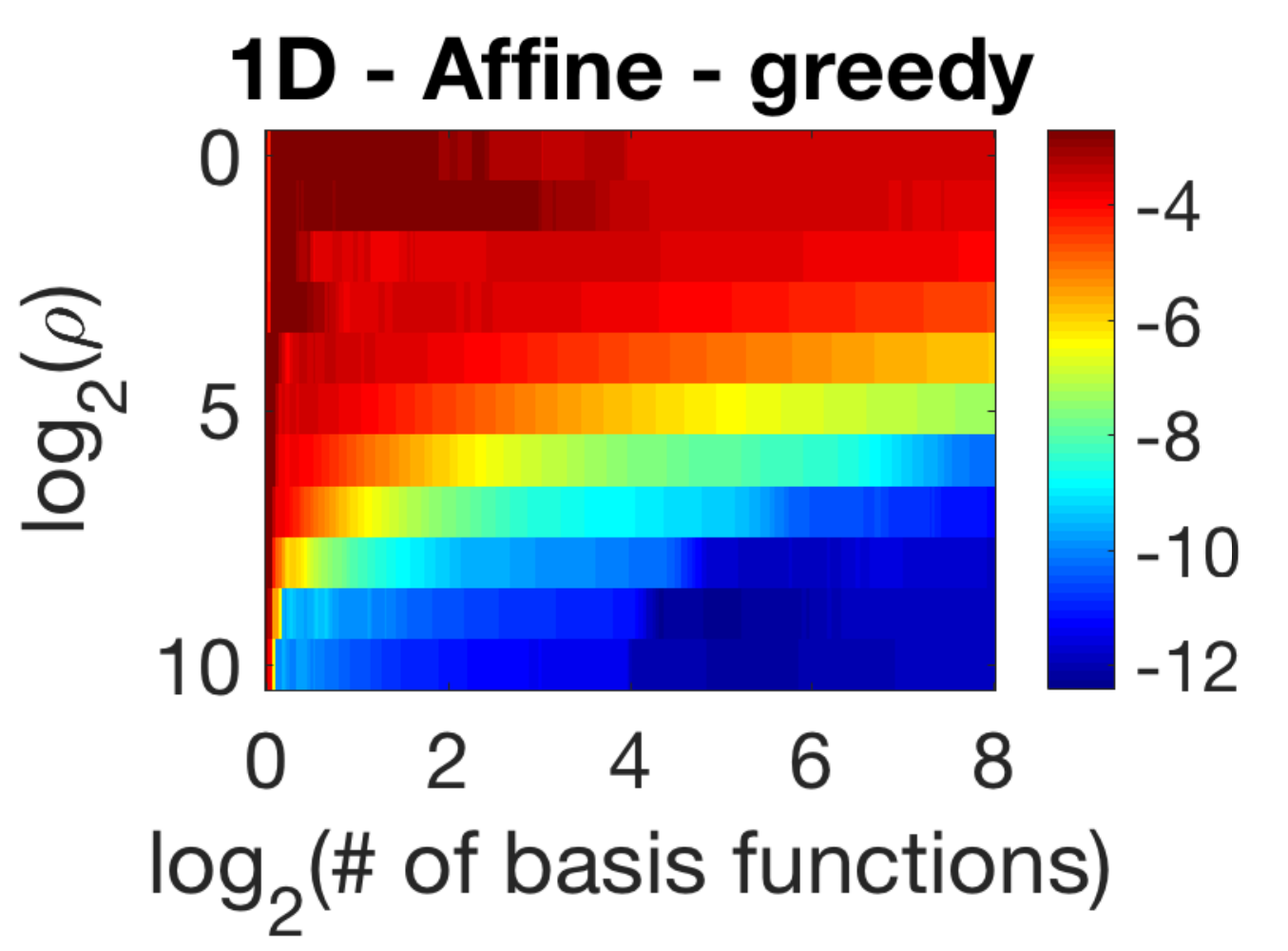}  
\hspace*{-2.5cm}\end{center}

\vspace*{-.45cm}

\caption{Approximation error of a function $V$ as a function of $\rho$ and the number of basis functions, for piecewise constant functions and piecewise affine functions. One-dimensional case (with a convex optimal value function). Left: fixed, right: greedy.
\label{fig:cvx_1d_plots}}
\end{figure}

\subsection{Two-dimensional}

\label{app:exp_2d}

We consider control problems on $[0,1]^d$, with $d=2$, with $2d$ potential discrete actions (one in each coordinate direction).
The corresponding Hamilton-Jacobi equation is 
\[ V(x) \log \eta +  \max_{i \in \{1,\dots,d\}}\Big|\frac{\partial V}{\partial x_i}(x) \Big|  + b(x)= 0,  \]
with natural boundary conditions. We consider the following  functions:
\BIT
\item $V(x_1,x_2) = (1  - 3 x_1)_+ + (6x_1 - 4  )_+  $. In the main paper, we consider this value function (with potential for variable selection) plotted in \myfig{plot_2d_sparse}, with performance plots in the bottom of \myfig{perf_1d}.

\item $V(x_1,x_2) = (1  - 3 x_1)_+ + (6x_1 - 4  )_+ + (1  - 3 x_2)_+ + (6x_2 - 4  )_+ $.
This function, without potential for variable selection, is plotted in \myfig{plot_2d_nonsparse}, with performance plots in \myfig{Perf_2d_non_sparse}.

\EIT
We consider $\eta \approx 0.919$, a number of nodes per dimension equal to $\tilde{S} = 45 \approx 2^{11/2}$ and thus a total number of nodes equal to $S = \tilde{S}^2 = 2025$, such that $\gamma \approx 0.9981$ and $\tau \approx 521$ (like in the one-dimensional case).

 \begin{figure}
\begin{center}
\hspace*{-.5cm}
\includegraphics[width=6.37cm]{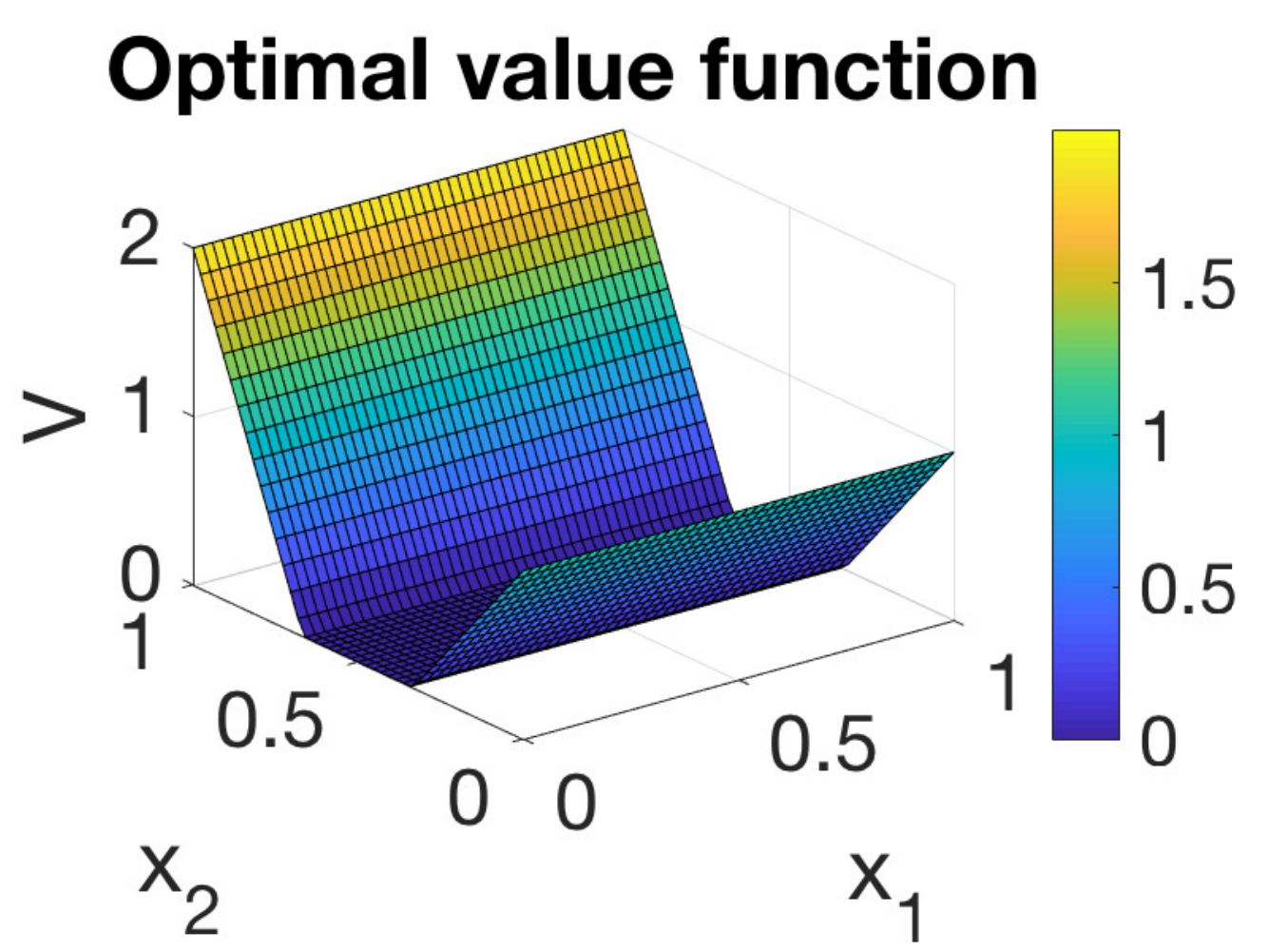} \hspace*{1cm}
 \includegraphics[width=6.37cm]{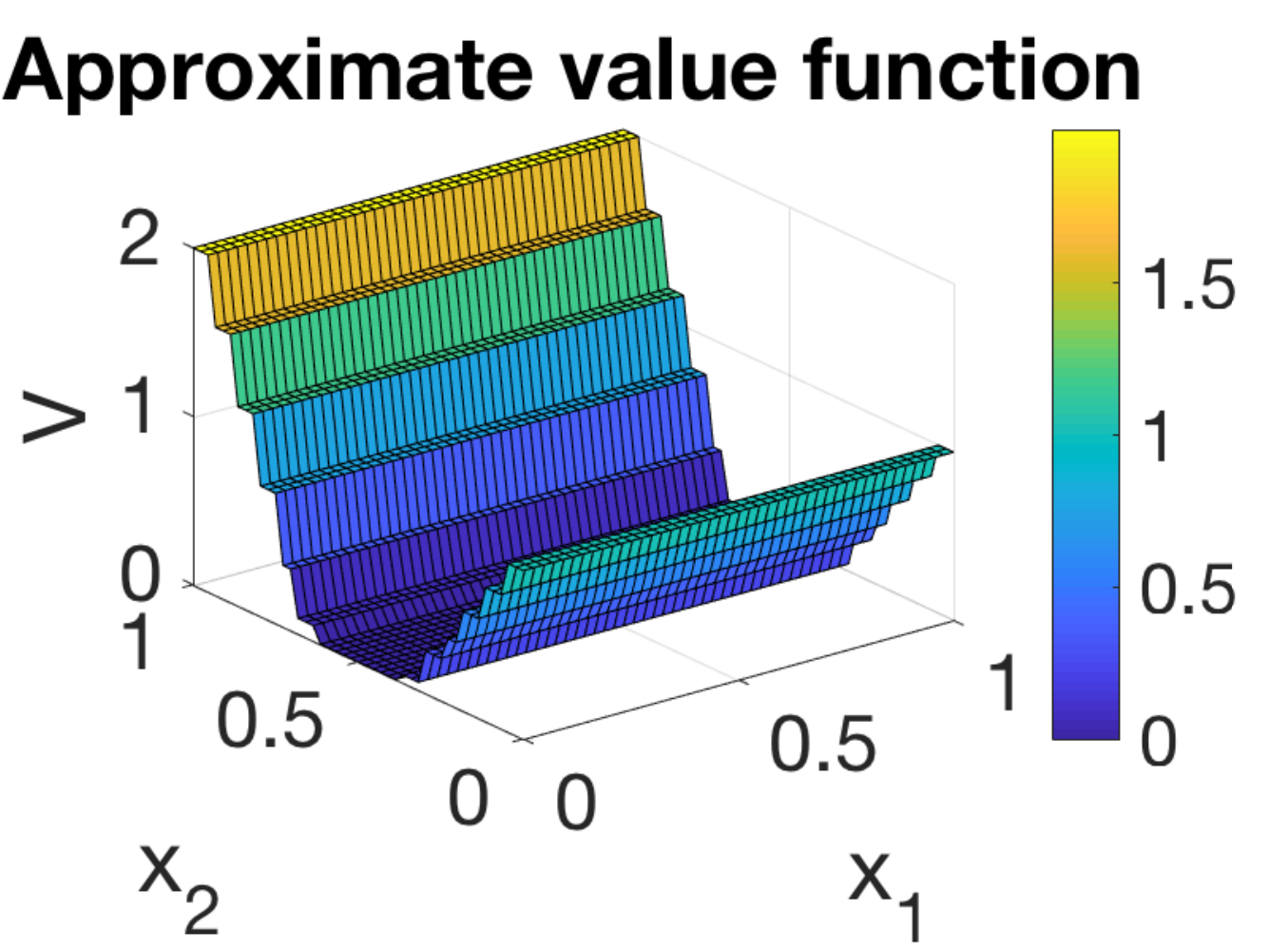}  
\hspace*{-2.5cm}\end{center}

\vspace*{-.45cm}

\caption{Value function with potential for variable selection. Left: optimal, right: approximation with piecewise constant functions.
\label{fig:plot_2d_sparse}}
\end{figure}

 \begin{figure}
\begin{center}
\hspace*{-.5cm}
\includegraphics[width=7.37cm]{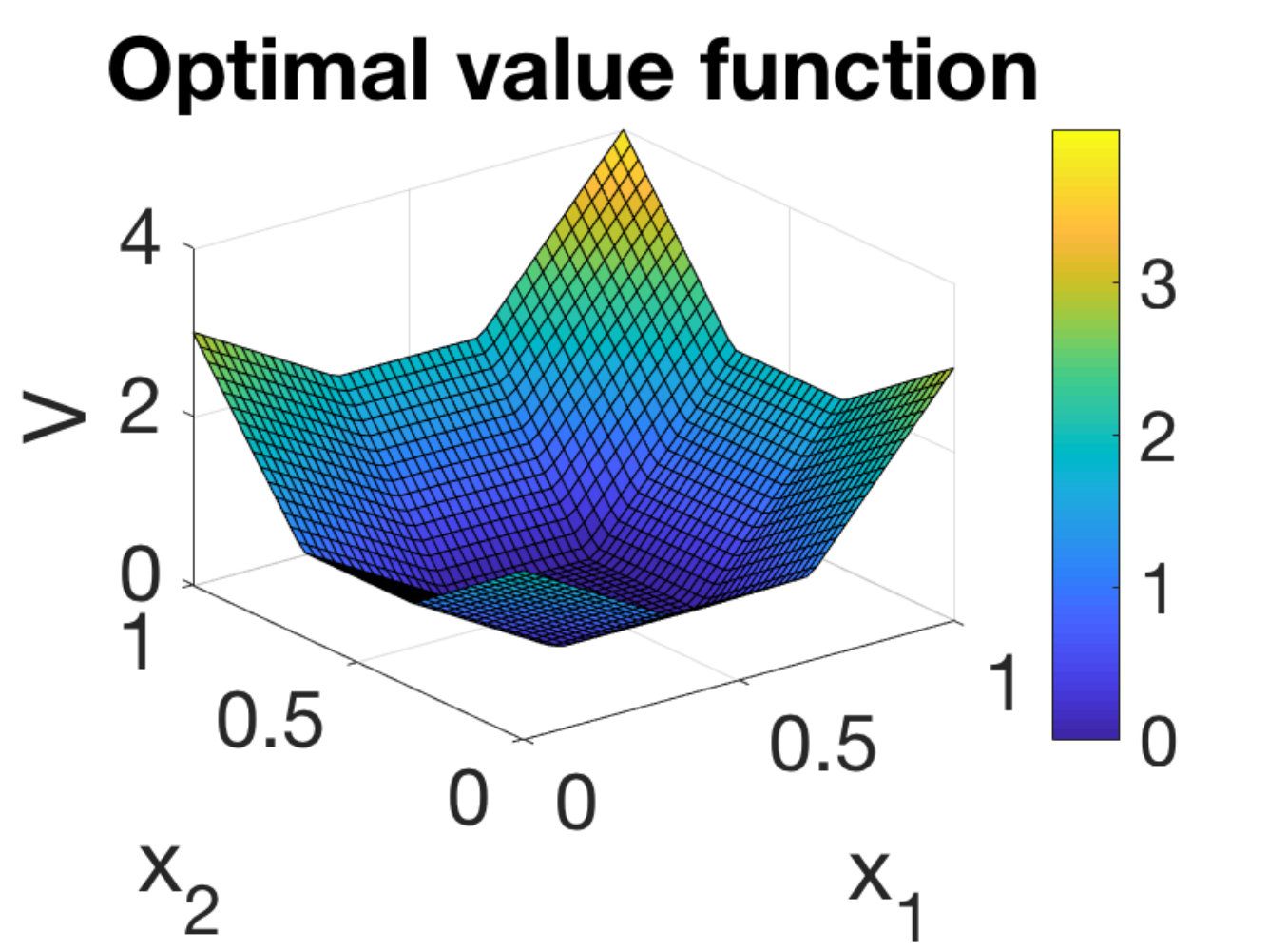} \hspace*{1cm}
 \includegraphics[width=7.37cm]{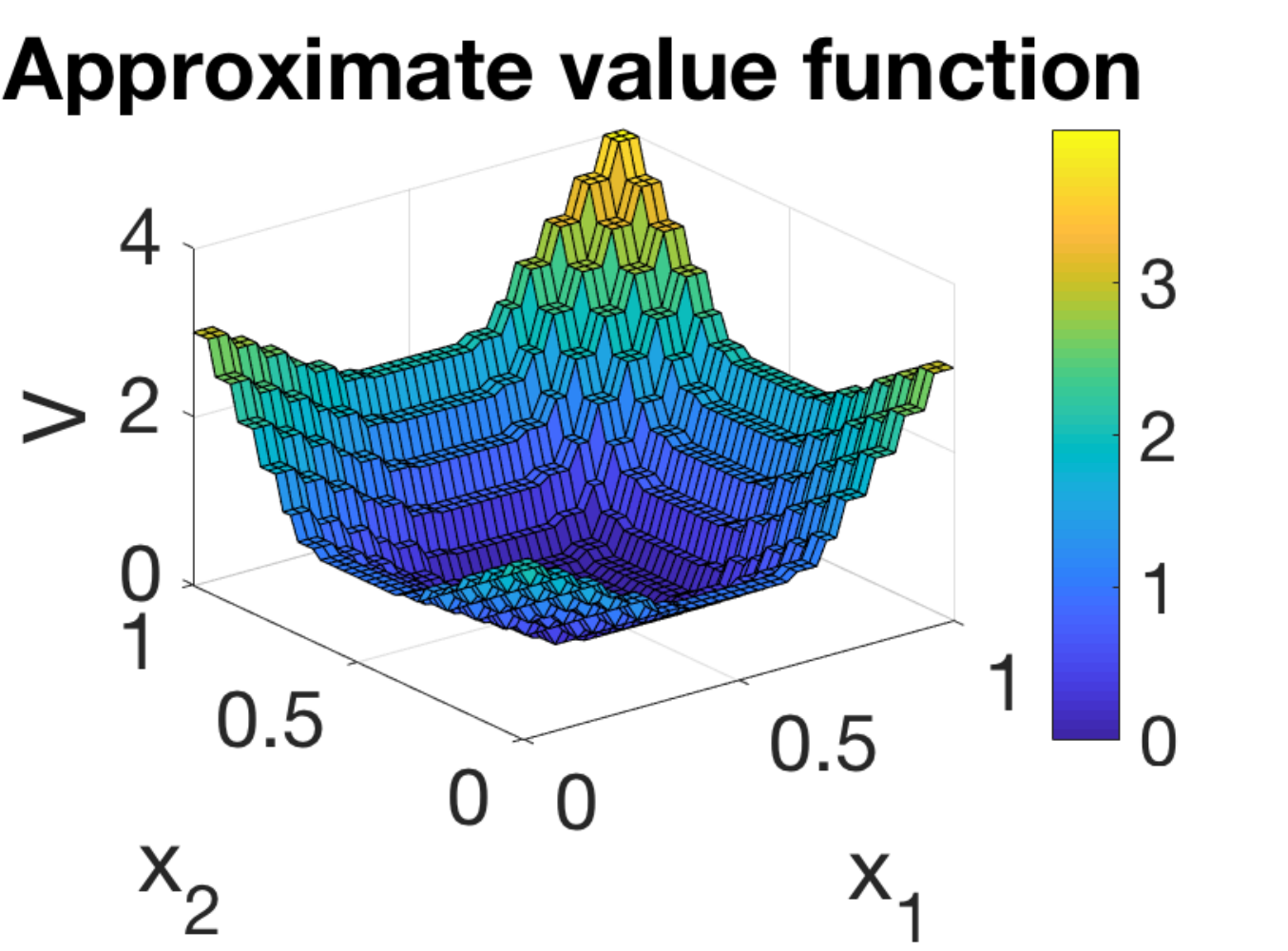}  
\hspace*{-2.5cm}\end{center}

\vspace*{-.45cm}

\caption{Value function without potential for variable selection. Left: optimal, right: approximation with piecewise constant functions.
\label{fig:plot_2d_nonsparse}}
\end{figure}

 \begin{figure}
\begin{center}
\hspace*{-.5cm}
\hspace*{-2.5cm}
\includegraphics[width=3.8cm]{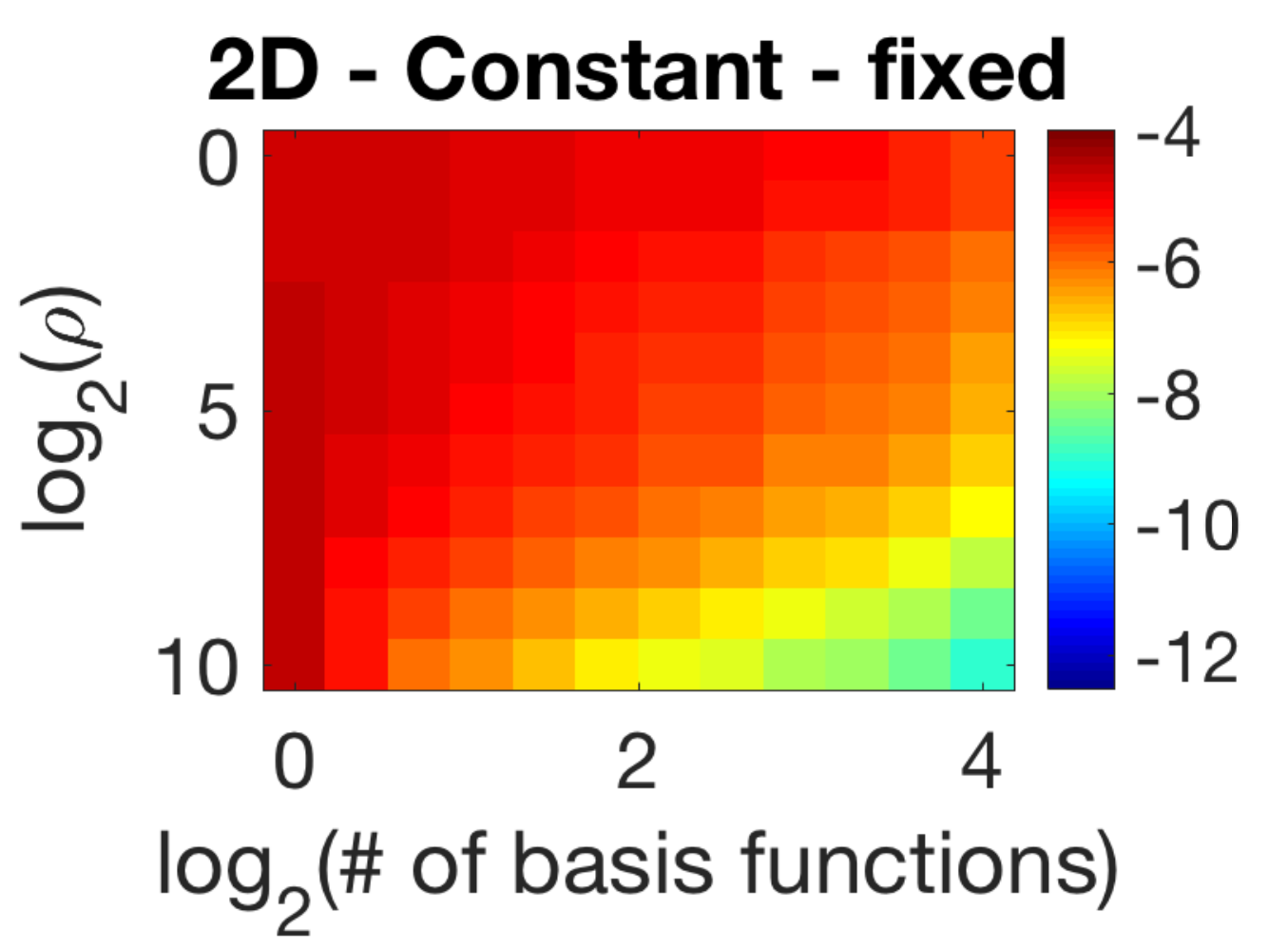} \hspace*{.1234cm}
\includegraphics[width=3.8cm]{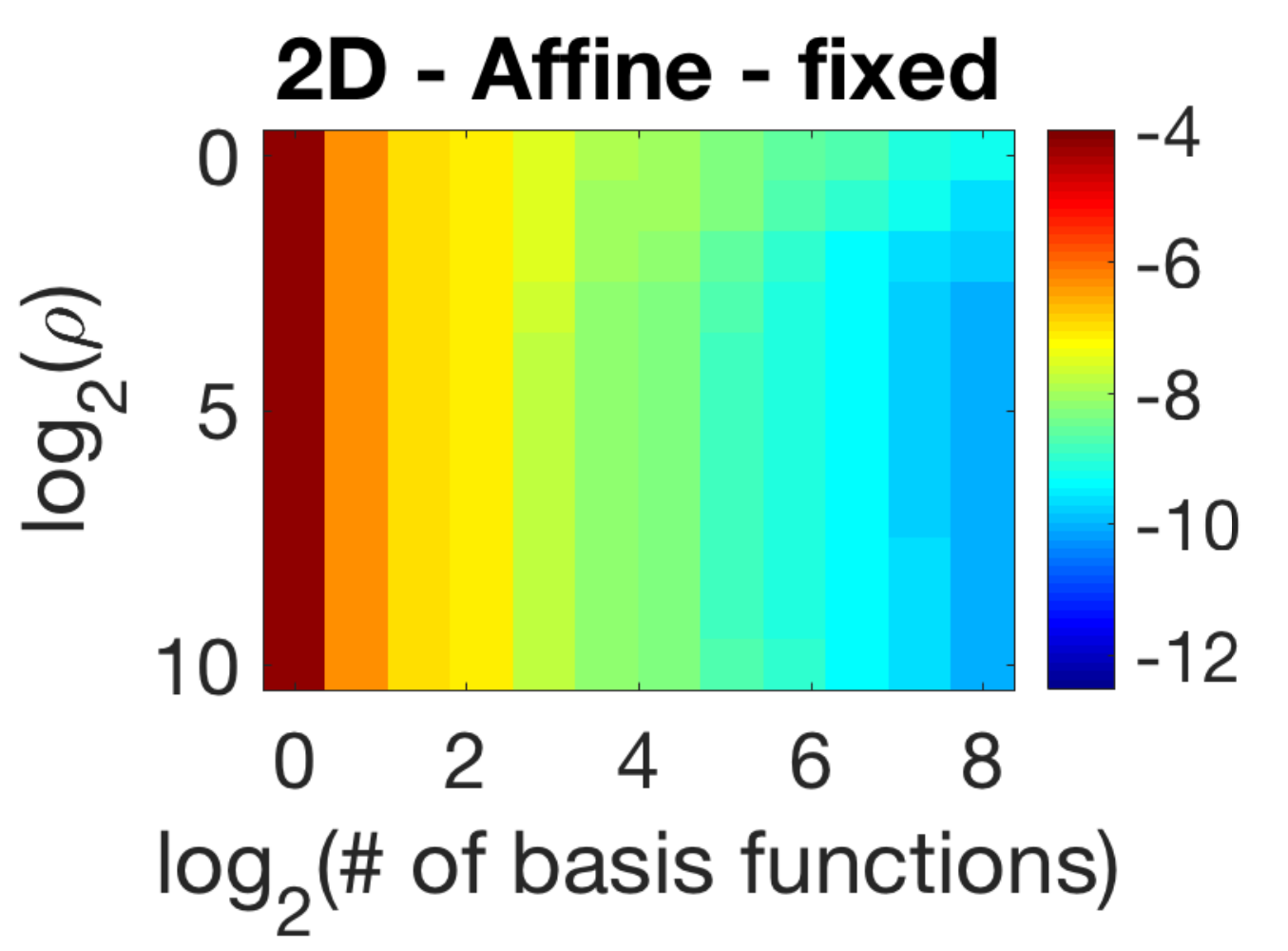} \hspace*{.1234cm}
\includegraphics[width=3.8cm]{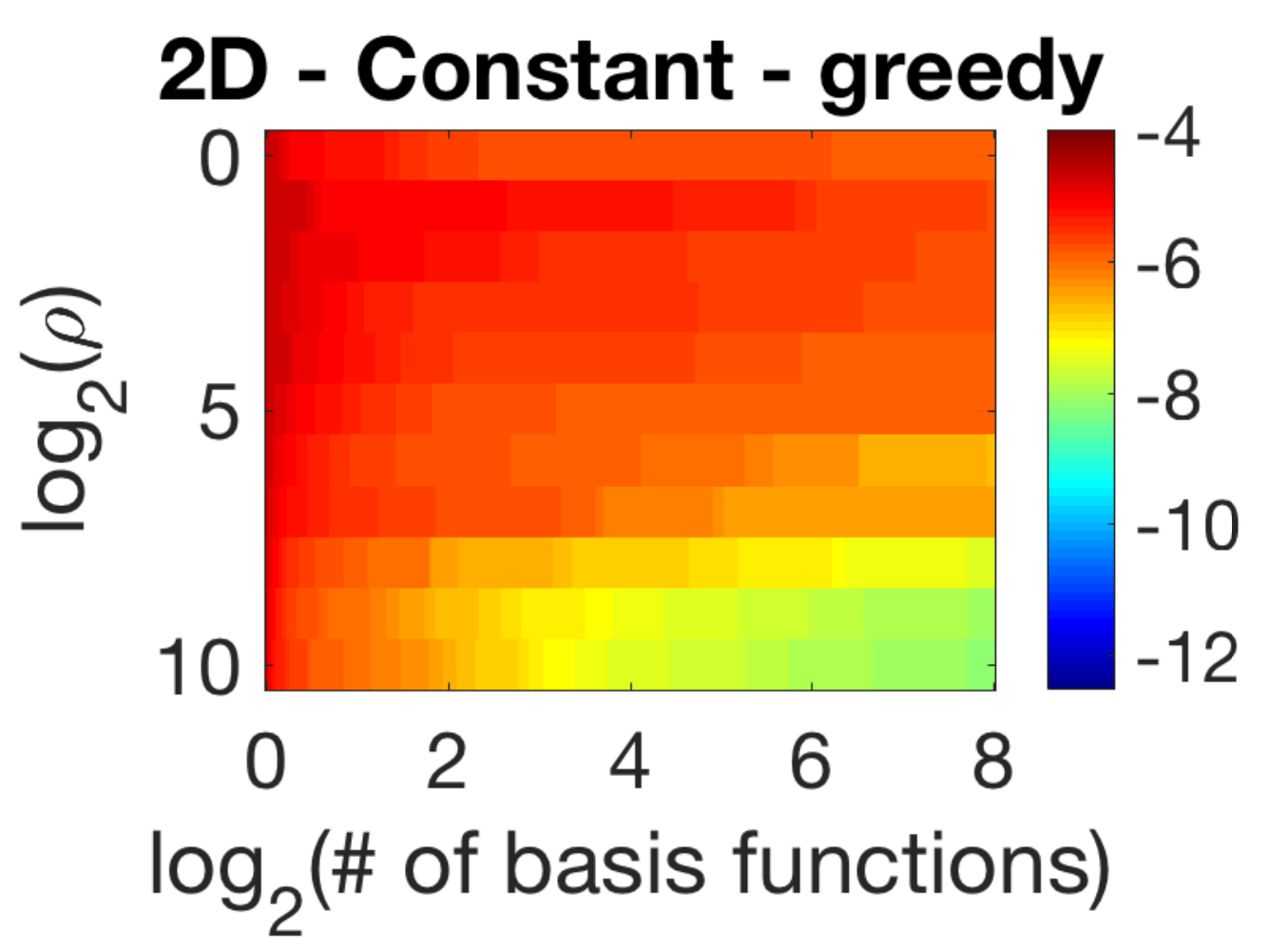} \hspace*{.1234cm}
\includegraphics[width=3.8cm]{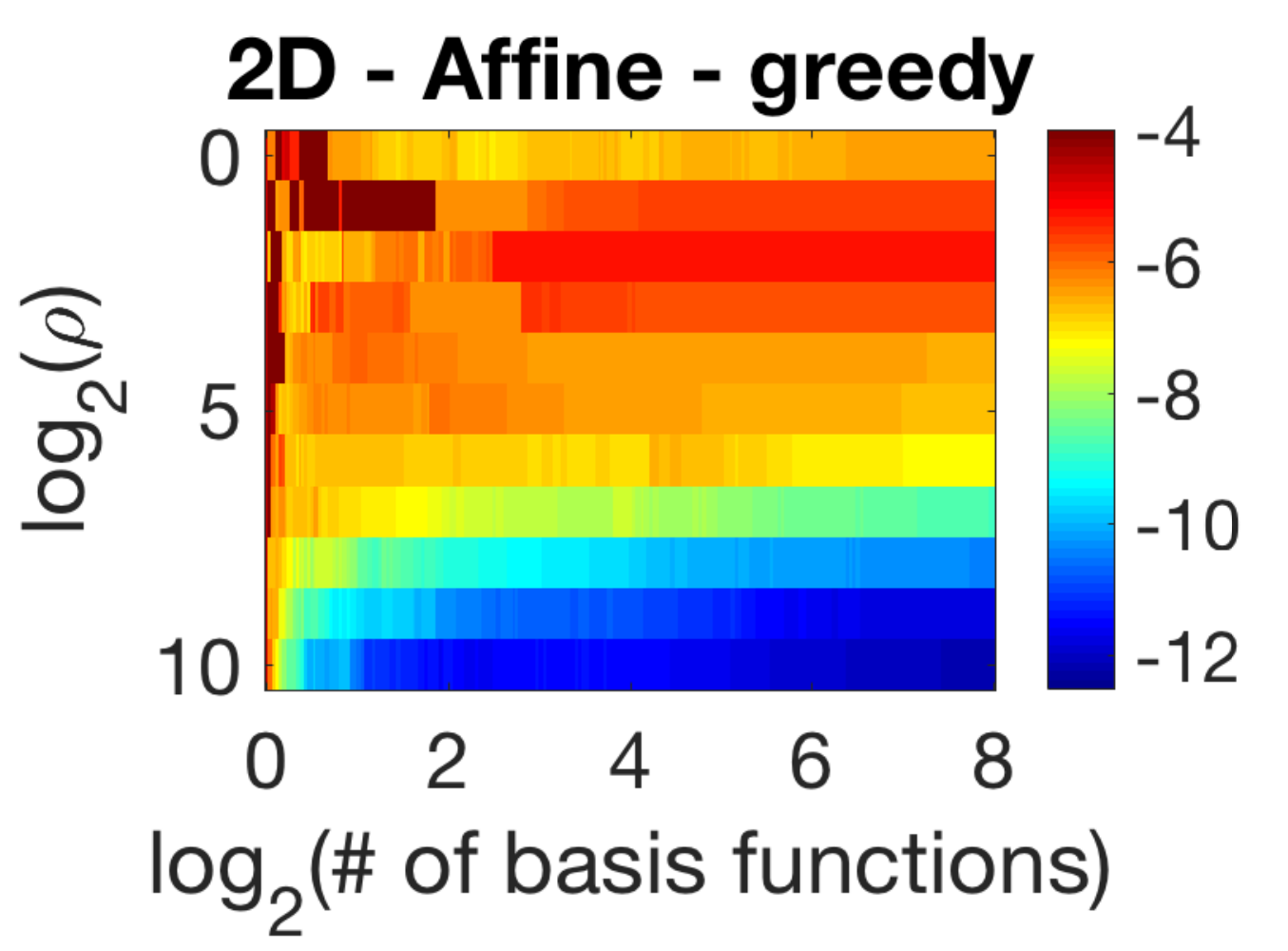}  
\hspace*{-2.5cm}
\hspace*{-2.5cm}\end{center}

\vspace*{-.45cm}

\caption{Approximation error of a function $V$ as a function of $\rho$ and the number of basis functions, for piecewise constant functions and piecewise affine functions. Two-dimensional case corresponding to the function in \myfig{plot_2d_nonsparse}. Left: fixed, right: greedy.  Compared to the bottom of \myfig{perf_1d} in the main paper, the gains in performance of the greedy method are not as large because the optimal reward function depends on all variables. \label{fig:Perf_2d_non_sparse}}
\end{figure}

\end{document}